\documentclass{article}

\usepackage{graphicx}
\usepackage{amsmath}
\usepackage{subfigure}
\newtheorem{theorem}{Theorem}
\newenvironment{proof}{{\noindent\it Proof Sketch.}}{\hfill $\square$\par}

\newtheorem{corollary}{Corollary}
\usepackage{booktabs}
\usepackage{wrapfig}
\usepackage{amsfonts}
\usepackage{courier}
\usepackage{algorithm}
\usepackage{algorithmic}
\usepackage{multirow}

\usepackage{fullpage}
\usepackage{authblk}
\usepackage[round]{natbib} 

\title{\huge \textbf {Reinforcement Learning with Dynamic Boltzmann Softmax Updates} }

\author{
  Ling Pan$^{1}$, Qingpeng Cai$^1$, Qi Meng$^2$, Wei Chen$^2$, Longbo Huang$^1$, Tie-Yan Liu$^2$ \\
  $^1$IIIS, Tsinghua University\\
  $^2$Microsoft Research Asia\\
}

\date{}

\begin{document}
\maketitle

\begin{abstract}
Value function estimation is an important task in reinforcement learning, i.e., prediction. The Boltzmann softmax operator is a natural value estimator and can provide several benefits. However, it does not satisfy the non-expansion property, and its direct use may fail to converge even in value iteration. In this paper, we propose to update the value function with dynamic Boltzmann softmax (DBS) operator, which has good convergence property in the setting of planning and learning. Experimental results on GridWorld show that the DBS operator enables better estimation of the value function, which rectifies the convergence issue of the softmax operator. Finally, we propose the DBS-DQN algorithm by applying dynamic Boltzmann softmax updates in deep Q-network, which outperforms DQN substantially in 40 out of 49 Atari games.
\end{abstract}

\section{Introduction}
Reinforcement learning has achieved groundbreaking success for many decision making problems, including robotics\cite{kober2013reinforcement}, game playing\cite{mnih2015human,silver2017mastering}, and many others.
Without full information of transition dynamics and reward functions of the environment,  the agent learns an optimal policy by interacting with the environment from experience.

Value function estimation is an important task in reinforcement learning, i.e., prediction \cite{sutton1988learning,d2016estimating,xu2018meta}.
In the prediction task, it requires the agent to have a good estimate of the value function in order to update towards the true value function.
A key factor to prediction is the action-value summary operator. 
The action-value summary operator for a popular off-policy method, Q-learning \cite{watkins1989learning}, is the hard max operator, which always commits to the maximum action-value function according to current estimation for updating the value estimator.
This results in pure exploitation of current estimated values and lacks the ability to consider other potential actions-values.
The ``hard max'' updating scheme may lead to misbehavior due to noise in stochastic environments \cite{hasselt2010double,van2013estimating,fox2015taming}. 
Even in deterministic environments, this may not be accurate as the value estimator is not correct in the early stage of the learning process.
Consequently, choosing an appropriate action-value summary operator is of vital importance.

The Boltzmann softmax operator is a natural value estimator \cite{sutton1998introduction,azar2012dynamic,cesa2017boltzmann} based on the Boltzmann softmax distribution, which is a natural scheme to address the exploration-exploitation dilemma and has been widely used in reinforcement learning \cite{sutton1998introduction,azar2012dynamic,cesa2017boltzmann}.
In addition, the Boltzmann softmax operator also provides benefits for reducing overestimation and gradient noise in deep Q-networks \cite{song2018revisiting}.
However, despite the advantages, it is challenging to apply the Boltzmann softmax operator in value function estimation.
As shown in \cite{littman1996generalized,asadi2017alternative}, the Boltzmann softmax operator is not a non-expansion, which may lead to multiple fixed-points and thus the optimal value function of this policy is not well-defined.
Non-expansion is a vital and widely-used sufficient property to guarantee the convergence of the planning and learning algorithm.
Without such property, the algorithm may misbehave or even diverge.

We propose to update the value function using the dynamic Boltzmann softmax (DBS) operator with good convergence guarantee.
The idea of the DBS operator is to make the parameter $\beta$ time-varying while being state-independent.
We prove that having $\beta_t$ approach $\infty$ suffices to guarantee the convergence of value iteration with the DBS operator. 
Therefore, the DBS operator rectifies the convergence issue of the Boltzmann softmax operator with fixed parameters.
Note that we also achieve a tighter error bound for the fixed-parameter softmax operator in general cases compared with \cite{song2018revisiting}.
In addition, we show that the DBS operator achieves good convergence rate. 

Based on this theoretical guarantee, we apply the DBS operator to estimate value functions in the setting of model-free reinforcement learning without known model. 
We prove that the corresponding DBS Q-learning algorithm also guarantees convergence.
Finally, we propose the DBS-DQN algorithm, which generalizes our proposed DBS operator from tabular Q-learning to deep Q-networks using function approximators in high-dimensional state spaces.

It is crucial to note the DBS operator is the only one that meets all desired properties proposed in \cite{song2018revisiting} up to now, as it ensures Bellman optimality, enables overestimation reduction, directly represents a policy, can be applicable to double Q-learning \cite{hasselt2010double}, and requires no tuning.

To examine the effectiveness of the DBS operator, we conduct extensive experiments to evaluate the effectiveness and efficiency.
We first evaluate DBS value iteration and DBS Q-learning on a tabular game, the GridWorld.
Our results show that the DBS operator leads to smaller error and better performance than vanilla Q-learning and soft Q-learning \cite{haarnoja2017reinforcement}.
We then evaluate DBS-DQN on large scale Atari2600 games, and we show that DBS-DQN outperforms DQN in 40 out of 49 Atari games.

The main contributions can be summarized as follows:
\begin{itemize}
\item Firstly, we analyze the error bound of the Boltzmann softmax operator with arbitrary parameters, including static and dynamic. 
\item Secondly, we propose the dynamic Boltzmann softmax (DBS) operator, which has good convergence property in the setting of planning and learning. 
\item Thirdly, we conduct extensive experiments to verify the effectiveness of the DBS operator in a tabular game and a suite of 49 Atari video games. Experimental results verify our theoretical analysis and demonstrate the effectiveness of the DBS operator.
\end{itemize}

\section{Preliminaries}
A Markov decision process (MDP) is defined by a 5-tuple $(\mathcal{S}, \mathcal{A}, p, r, \gamma)$,
where $\mathcal{S}$ and $\mathcal{A}$ denote the set of states and actions,
$p(s'|s, a)$ represents the transition probability from state $s$ to state $s'$ under action $a$,
and $r(s, a)$ is the corresponding immediate reward.
The discount factor is denoted by $\gamma \in [0, 1)$, which controls the degree of importance of future rewards.

At each time, the agent interacts with the environment with its policy $\pi$, a mapping from state to action.
The objective is to find an optimal policy that maximizes the expected discounted long-term reward $\mathbb{E}[\sum_{t=0}^{\infty} \gamma^t r_t | \pi]$, which can be solved by estimating value functions.
The state value of $s$ and state-action value of $s$ and $a$ under policy $\pi$ are defined as $V^{\pi}(s) = \mathbb{E}_{\pi}[\sum_{t=0}^{\infty} \gamma^t r_t | s_0=s]$ and $Q^{\pi}(s, a) = \mathbb{E}_{\pi} [\sum_{t=0}^{\infty} \gamma^t r_t | s_0=s, a_0=a]$.
The optimal value functions are defined as $V^*(s) = \max_{\pi} V^{\pi}(s)$ and $Q^*(s,a) = \max_{\pi} Q^{\pi}(s,a)$.

The optimal value function $V^*$ and $Q^*$ satisfy the Bellman equation, which is defined recursively as in Eq. (\ref{eq:bellman}):
\begin{equation}
\begin{split}
& V^*(s) = \max_{a \in A} Q^*(s,a), \\
& Q^*(s,a) = r(s,a) + \gamma \sum_{s' \in S} p(s'|s,a) V^*(s').
\end{split}
\label{eq:bellman}
\end{equation}
Starting from an arbitrary initial value function $V_0$, the optimal value function $V^*$ can be computed by value iteration \cite{bellman1957dynamic} according to an iterative update $V_{k+1} = \mathcal{T} V_{k}$, where $\mathcal{T}$ is the Bellman operator defined by
\begin{equation}
(\mathcal{T}V)(s) = \max_{a \in A} \Big[ r(s,a) + \sum_{s' \in S} p(s'|s,a) \gamma V(s') \Big].
\end{equation}

When the model is unknown, Q-learning \cite{watkins1992q} is an effective algorithm to learn by exploring the environment.
Value estimation and update for a given trajectory $(s,a,r,s')$ for Q-learning is defined as:
\begin{equation}
Q(s,a) = (1-\alpha) Q(s,a) + \alpha \left( r + \gamma \max_{a'} Q(s',a') \right),
\label{eq:q_update}
\end{equation}
where $\alpha$ denotes the learning rate.
Note that Q-learning employs the hard max operator for value function updates, i.e.,
\begin{equation}
\max({\bf{X}}) = \max_i x_i.
\end{equation}

Another common operator is the log-sum-exp operator \cite{haarnoja2017reinforcement}:
\begin{equation}
L_{\beta}({\bf{X}}) = \frac{1}{\beta} \log(\sum_{i=1}^n e^{\beta x_i}).
\end{equation}
The Boltzmann softmax operator is defined as:
\begin{equation}
{\rm boltz}_{\beta}({\bf{X}}) = \frac{ \sum_{i=1}^n e^{\beta x_i} x_i} { \sum_{i=1}^n e^{\beta x_i} }.
\end{equation}

\section{Dynamic Boltzmann Softmax Updates} \label{sec:dbs}
In this section, we propose the dynamic Boltzmann softmax operator (DBS) for value function updates.
We show that the DBS operator does enable the convergence in value iteration, and has good convergence rate guarantee.
Next, we show that the DBS operator can be applied in Q-learning algorithm, and also ensures the convergence.

The DBS operator is defined as: $\forall s \in \mathcal{S}$, 
\begin{equation}
{\rm{boltz}}_{\beta_t}(Q(s, \cdot)) = \frac{\sum_{a \in \mathcal{A}} e^{\beta_t Q(s,a)} Q(s,a)}{\sum_{a \in \mathcal{A}} e^{\beta_t Q(s,a)}},
\end{equation}
where $\beta_t$ is non-negative.
Our core idea of the DBS operator $\rm{boltz}_{\beta_t}$ is to dynamically adjust the value of $\beta_t$ during the iteration. 

We now give theoretical analysis of the proposed DBS operator and show that it has good convergence guarantee.

\subsection{Value Iteration with DBS Updates}
Value iteration with DBS updates admits a time-varying, state-independent sequence $\{\beta_t\}$ and updates the value function according to the DBS operator $\rm{boltz}_{\beta_t}$ by iterating the following steps:
\begin{equation}
\label{dbs_vi}
\begin{split}
&\forall s,a, Q_{t+1}(s, a) \leftarrow  \sum_{s'} p(s' | s,a) \left[ r(s,a) + \gamma V_t(s') \right] \\
&\forall s, V_{t+1}(s) \leftarrow  {\rm{boltz}}_{\beta_t}(Q_{t+1}(s, \cdot))
\end{split}
\end{equation}
For the ease of the notations, we denote $\mathcal{T}_{\beta_t} $ the function that iterates any value function by Eq. (\ref{dbs_vi}). 

Thus, the way to update the value function is according to the exponential weighting scheme which is related to both the current estimator and the parameter $\beta_t$.

\subsubsection{Theoretical Analysis}
It has been shown that the Boltzmann softmax operator is not a non-expansion \cite{littman1996generalized}, as it does not satisfy Ineq. (\ref{eq:non_expansion}).
\begin{equation}
\begin{split}
&| {\rm boltz}_{\beta} (Q_1(s, \cdot)) - {\rm boltz}_{\beta} (Q_2(s, \cdot)) | \\
\leq &\max_{a}| Q_1(s,a) - Q_2(s,a)|, \ \forall s \in \mathcal{S}.
\end{split}
\label{eq:non_expansion}
\end{equation}
Indeed, the non-expansion property is a vital and widely-used sufficient condition for achieving convergence of learning algorithms.
If the operator is not a non-expansion, the uniqueness of the fixed point may not be guaranteed, which can lead to misbehaviors in value iteration.

In Theorem \ref{tm:dbs_con}, we provide a novel analysis which demonstrates that the DBS operator enables the convergence of DBS value iteration to the optimal value function. 

\begin{theorem}
	{\bf{(Convergence of value iteration with the DBS operator)}}
	For any dynamic Boltzmann softmax operator $ \rm{boltz}_{\beta_t} $,
	if $\beta_t$ approaches $\infty$,
	 the value function after $t$ iterations $V_t$ converges to the optimal value function $V^*$.
	\label{tm:dbs_con}
\end{theorem}

\begin{proof}
By the same way as Eq. (\ref{dbs_vi}), let  $\mathcal{T}_m$ be the function that iterates any value function by the max operator.

Thus, we have
{\small \begin{equation}
\begin{split}
|| (\mathcal{T}_{\beta_t} V_1) - (\mathcal{T}_m V_2) ||_{\infty} & \leq \underbrace{|| (\mathcal{T}_{\beta_t} V_1) - (\mathcal{T}_m V_1) ||_{\infty}}_{(A)} \\
&+ \underbrace{|| (\mathcal{T}_m V_1) - (\mathcal{T}_m V_2) ||_{\infty}}_{(B)}
\label{eq:term_ab_sum}
\end{split}
\end{equation}}

For the term $(A)$, we have
{\small \begin{align}
& ||(\mathcal{T}_{\beta_t} V_1) - (\mathcal{T}_m V_1)||_{\infty}
\leq \frac{\log(|A|)}{\beta_t} \label{eq:boltz_lse_1}
\end{align}}

For the proof of the Ineq. (\ref{eq:boltz_lse_1}), please refer to the supplemental material.

For the term $(B)$, we have
{\small \begin{align}
& || (\mathcal{T}_m V_1) - (\mathcal{T}_m V_2) ||_{\infty} 
\leq \gamma ||V_1 - V_2||
\label{eq:term_b}
\end{align}}

Combining (\ref{eq:term_ab_sum}), (\ref{eq:boltz_lse_1}), and (\ref{eq:term_b}), we have
{\small \begin{equation}
|| (\mathcal{T}_{\beta_t} V_1) - (\mathcal{T}_m V_2) ||_{\infty} \leq \gamma ||V_1 - V_2||_{\infty} + \frac{\log(|A|)}{\beta_t}
\end{equation}}

As the max operator is a contraction mapping, then from Banach fixed-point theorem \cite{banach} we have
$\mathcal{T}_m V^* = V^*$.

By the definition of DBS value iteration in Eq. (\ref{dbs_vi}), 
{\small \begin{align}
& ||V_t - V^*||_{\infty} \\
=& || (\mathcal{T}_{\beta_t} ... \mathcal{T}_{\beta_1}) V_0 - (\mathcal{T}_m ... \mathcal{T}_m) V^*||_{\infty} \\
\leq& \gamma || (\mathcal{T}_{\beta_{t-1}} ... \mathcal{T}_{\beta_1}) V_0 - (\mathcal{T}_m ... \mathcal{T}_m) V^*||_{\infty} + \frac{\log(|A|)}{\beta_t} \\
\leq& \gamma^t ||V_0 - V^*||_{\infty} + \log(|A|) \sum_{k=1}^t \frac{\gamma^{t-k}}{\beta_k} \label{eq:vi_con_basic}
\end{align}}

If $\beta_t \to \infty$, then $\lim_{t \to \infty} \sum_{k=1}^t \frac{\gamma^{t-k}}{\beta_k} = 0$, where the full proof is referred to the supplemental material.

Taking the limit of the right hand side of Eq. (\ref{eq:vi_con_basic}), we obtain $\lim_{t \to \infty} ||V_{t+1} - V^*||_{\infty} = 0$.
\end{proof}

Theorem \ref{tm:dbs_con} implies that DBS value iteration does converge to the optimal value function if $\beta_t$ approaches infinity.
During the process of dynamically adjusting $\beta_t$, although the non-expansion property may be violated for some certain values of $\beta$, we only need the state-independent parameter $\beta_t$ to approach infinity to guarantee the convergence.

Now we justify that the DBS operator has good convergence rate guarantee, where the proof is referred to the supplemental material.
\begin{theorem}
{\bf{(Convergence rate of value iteration with the DBS operator)}} For any power series $\beta_t = t^p (p > 0)$, let $V_0$ be an arbitrary initial value function such that $||V_0||_{\infty}\leq \frac{R}{1-\gamma}$, where $R=\max_{s,a}|r(s,a)|$, 
	we have that for any non-negative $\epsilon < 1/4$, after
	$\max \{ O\big( \frac{\log(\frac{1}{\epsilon})  + \log(\frac{1}{1 - \gamma}) + \log (R)}{\log(\frac{1}{\gamma})}), O\big({( \frac{1}{(1 - \gamma) \epsilon})}^{\frac{1}{p}}\big) \}$
	steps, the error $||V_t - V^*||_{\infty} \leq \epsilon$.
\label{tm:dbs_con_rate}
\end{theorem}

For the larger value of $p$, the convergence rate is faster.
Note that when $p$ approaches $\infty$, the convergence rate is dominated by the first term, which has the same order as that of the standard Bellman operator, implying that the DBS operator is competitive with the standard Bellman operator in terms of the convergence rate in known environment.

From the proof techniques in Theorem \ref{tm:dbs_con}, we derive the error bound of value iteration with the Boltzmann softmax operator with fixed parameter $\beta$ in Corollary \ref{cor:bs_error_bound}, and the proof is referred to the supplemental material.

\begin{corollary}
{\bf{(Error bound of value iteration with Boltzmann softmax operator)}}
For any Boltzmann softmax operator with fixed parameter $\beta$, we have
\begin{equation}
\lim_{t\rightarrow \infty}||V_t - V^*||_{\infty} \leq \min\left\{\frac{\log(|A|)}{\beta (1 - \gamma)}, \frac{2R}{{(1-\gamma)}^{2}}\right\}.
\end{equation}
\label{cor:bs_error_bound}
\end{corollary}

Here, we show that after an infinite number of iterations, the error between the value function $V_t$ computed by the Boltzmann softmax operator with the fixed parameter $\beta$ at the $t$-th iteration and the optimal value function $V^*$ can be upper bounded.
However, although the error can be controlled, the direct use of the Boltzmann softmax operator with fixed parameter may introduce performance drop in practice, due to the fact that it violates the non-expansion property.

Thus, we conclude that the DBS operator performs better than the traditional Boltzmann softmax operator with fixed parameter in terms of convergence.

\subsubsection{Relation to Existing Results}
In this section, we compare the error bound in Corollary \ref{cor:bs_error_bound} with that in \cite{song2018revisiting}, which studies the error bound of the softmax operator with a fixed parameter $\beta$.

Different from \cite{song2018revisiting}, we provide a more general convergence analysis of the softmax operator covering both static and dynamic parameters.
We also achieve a tighter error bound when 
\begin{equation}
\beta \geq \frac{2}{\max \{ \frac{\gamma(|A|-1)}{\log(|A|)}, \frac{2 \gamma (|A|-1) R}{1-\gamma} \} -1},
\label{eq:beta_range}
\end{equation}
where $R$ can be normalized to $1$ and $|A|$ denotes the number of actions.
The term on the RHS of Eq. (\ref{eq:beta_range}) is quite small as shown in Figure \ref{fig:bound}(a), where we set $\gamma$ to be some commonly used values in $\{0.85, 0.9, 0.95, 0.99\}$.
The shaded area corresponds to the range of $\beta$ within our bound is tighter, which is a general case.

\begin{figure}[!h]
	\centering
	\subfigure[Range of $\beta$ within which our bound is tighter.
	]{
		\begin{minipage}[t]{0.5\linewidth}
			\centering
			\includegraphics[scale=0.35]{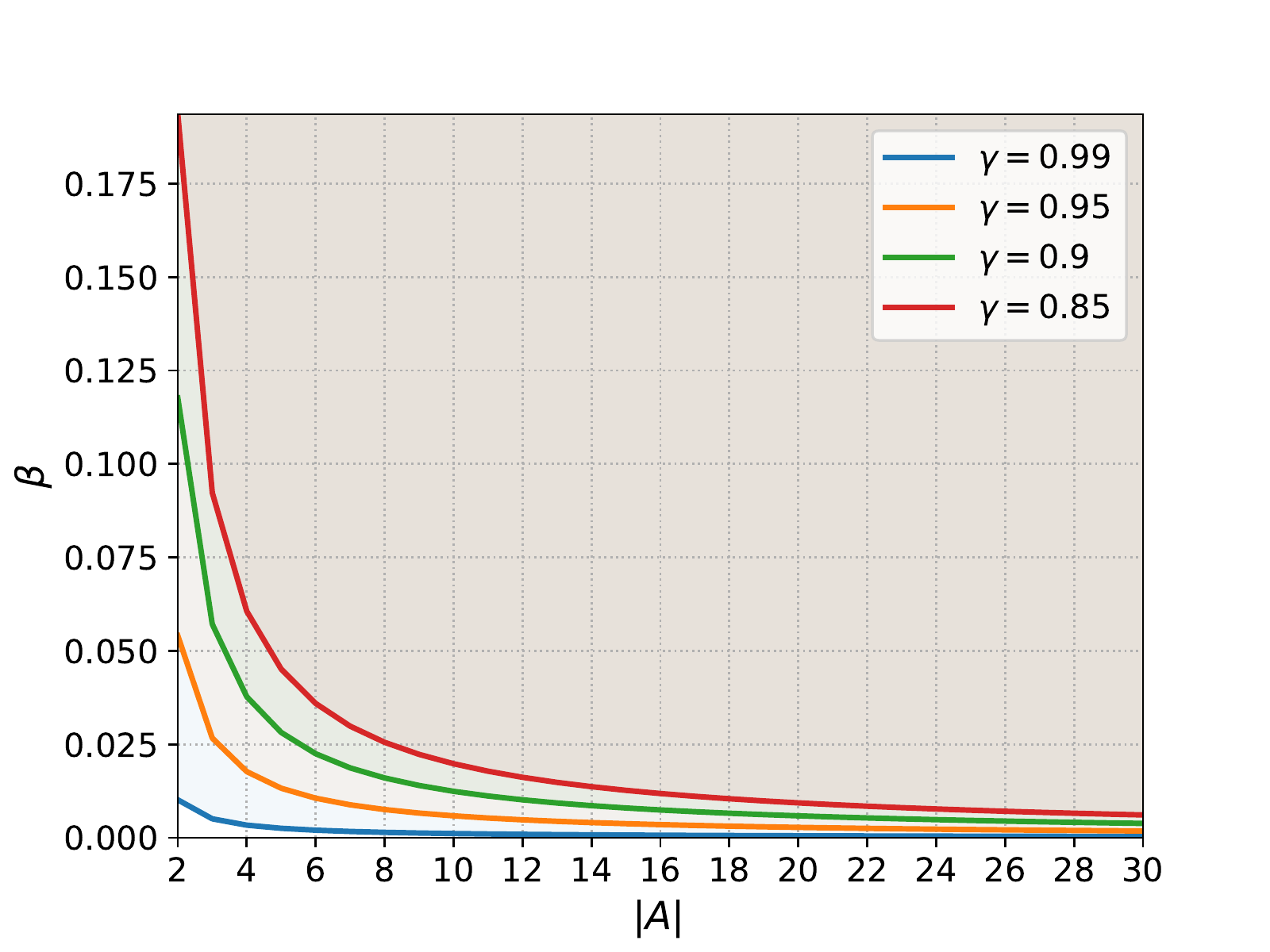}
		\end{minipage}
	}%
	\subfigure[Improvement ratio.
	]{
		\begin{minipage}[t]{0.5\linewidth}
			\centering
			\includegraphics[scale=0.35]{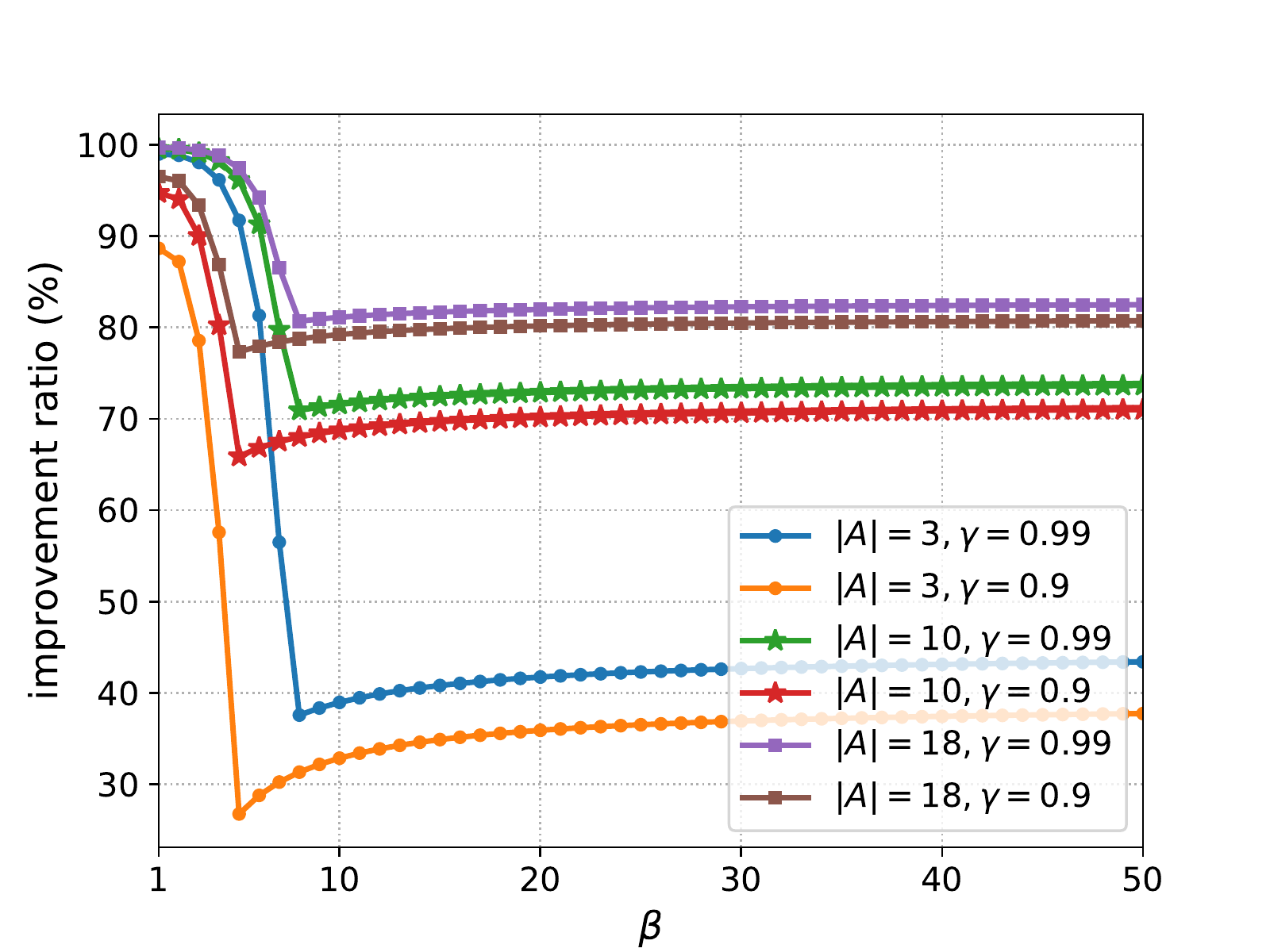}
		\end{minipage}
	}
	\caption{Error bound comparison.}
	\label{fig:bound}
\end{figure}

Please note that the case where $\beta$ is extremely small, i.e., approaches $0$, is usually not considered in practice.
Figure \ref{fig:bound}(b) shows the improvement of the error bound, which is defined as $\frac{\rm their \ bound - our \ bound}{\rm their \ bound} \times 100\%$.
Note that in the Arcade Learning Environment \cite{bellemare2013arcade}, $|A|$ is generally in $[3, 18]$.
Moreover, we also give an analysis of the convergence rate of the DBS operator.

\subsubsection{Empirical Results} \label{sec:vi_exp}
We first evaluate the performance of DBS value iteration to verify our convergence results in a toy problem, the GridWorld (Figure \ref{fig:vi_gridworld}(a)), which is a larger variant of the environment of \cite{o2016combining}.

\begin{figure}[!h]
	\subfigure[GridWorld.
	]{
		\begin{minipage}[t]{0.5\linewidth}
			\centering
			\includegraphics[scale=0.33]{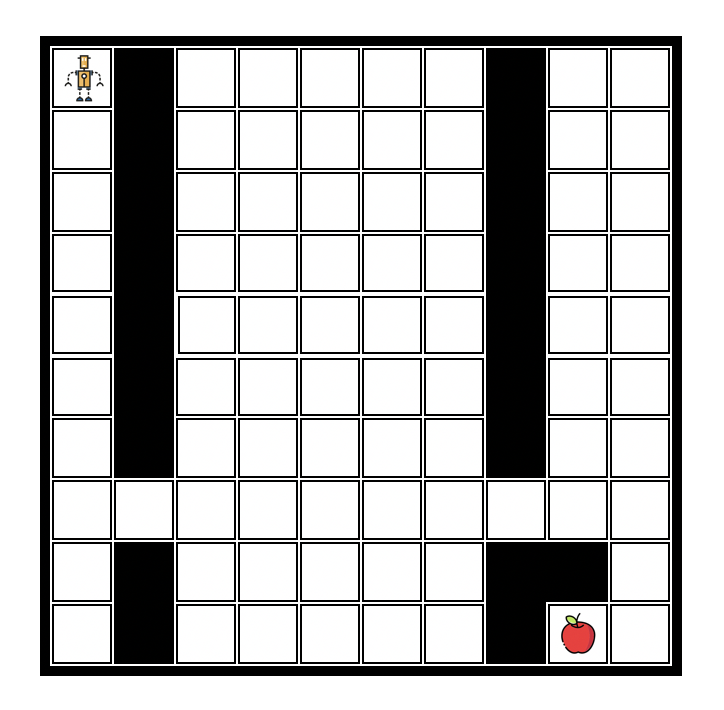}
		\end{minipage}
	}%
	\subfigure[Value loss.
	]{
		\begin{minipage}[t]{0.5\linewidth}
			\centering
			\includegraphics[scale=0.35]{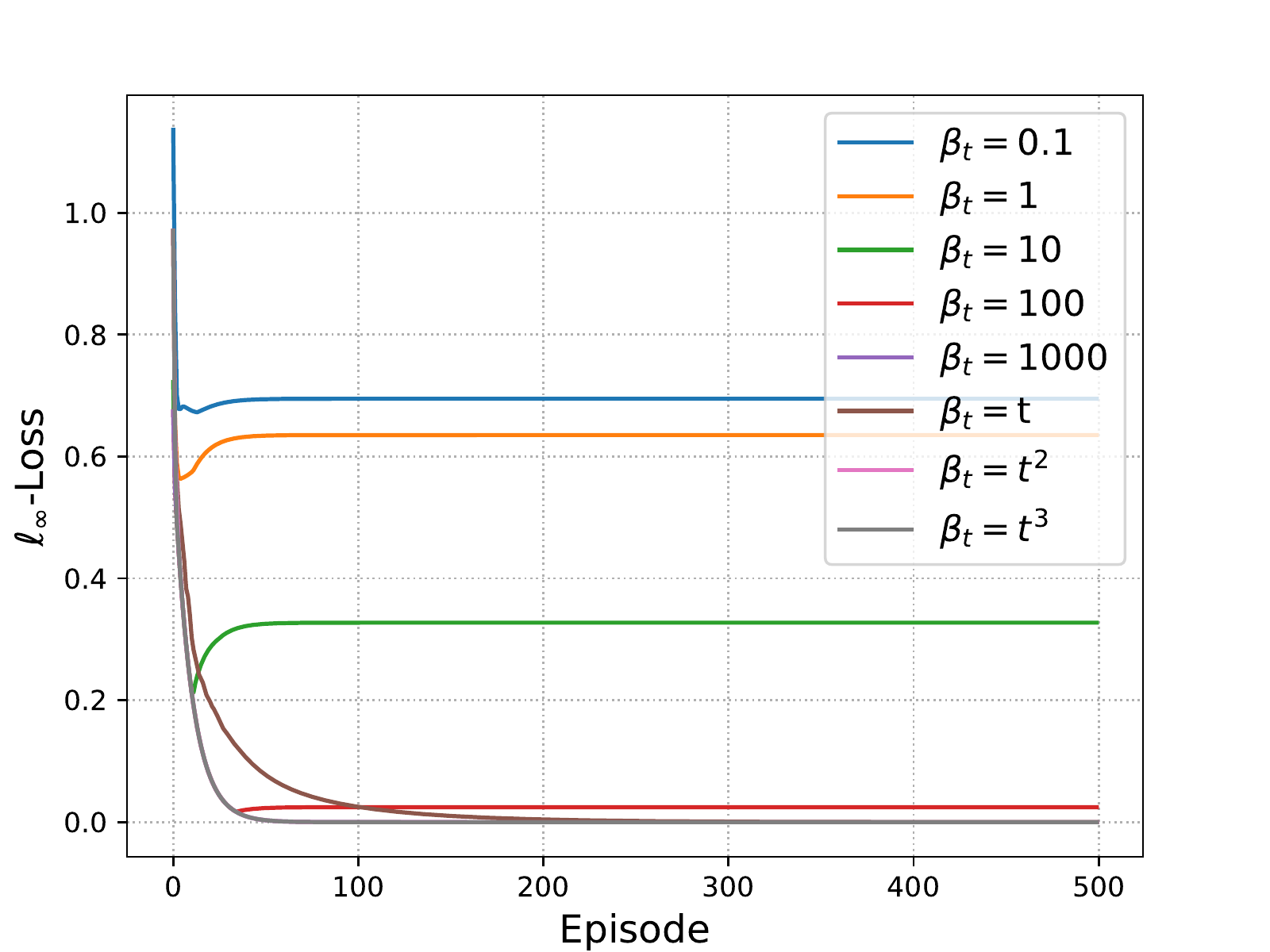}
		\end{minipage}
	}
	\subfigure[Value loss of the last episode in log scale.
	]{
		\begin{minipage}[t]{0.5\linewidth}
			\centering
			\includegraphics[scale=0.35]{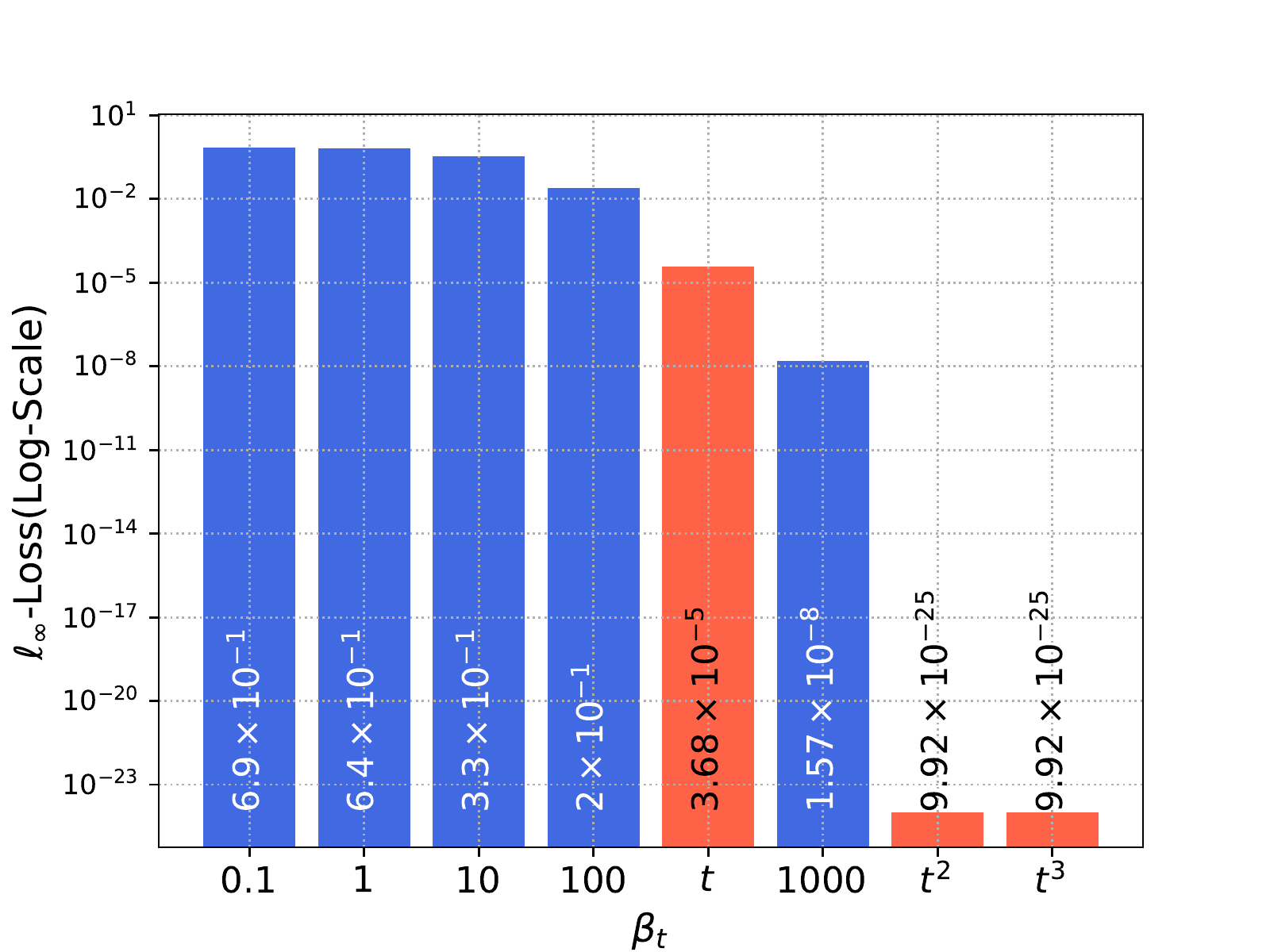}
		\end{minipage}
	}%
	\subfigure[Convergence rate.
	]{
		\begin{minipage}[t]{0.5\linewidth}
			\centering
			\includegraphics[scale=0.35]{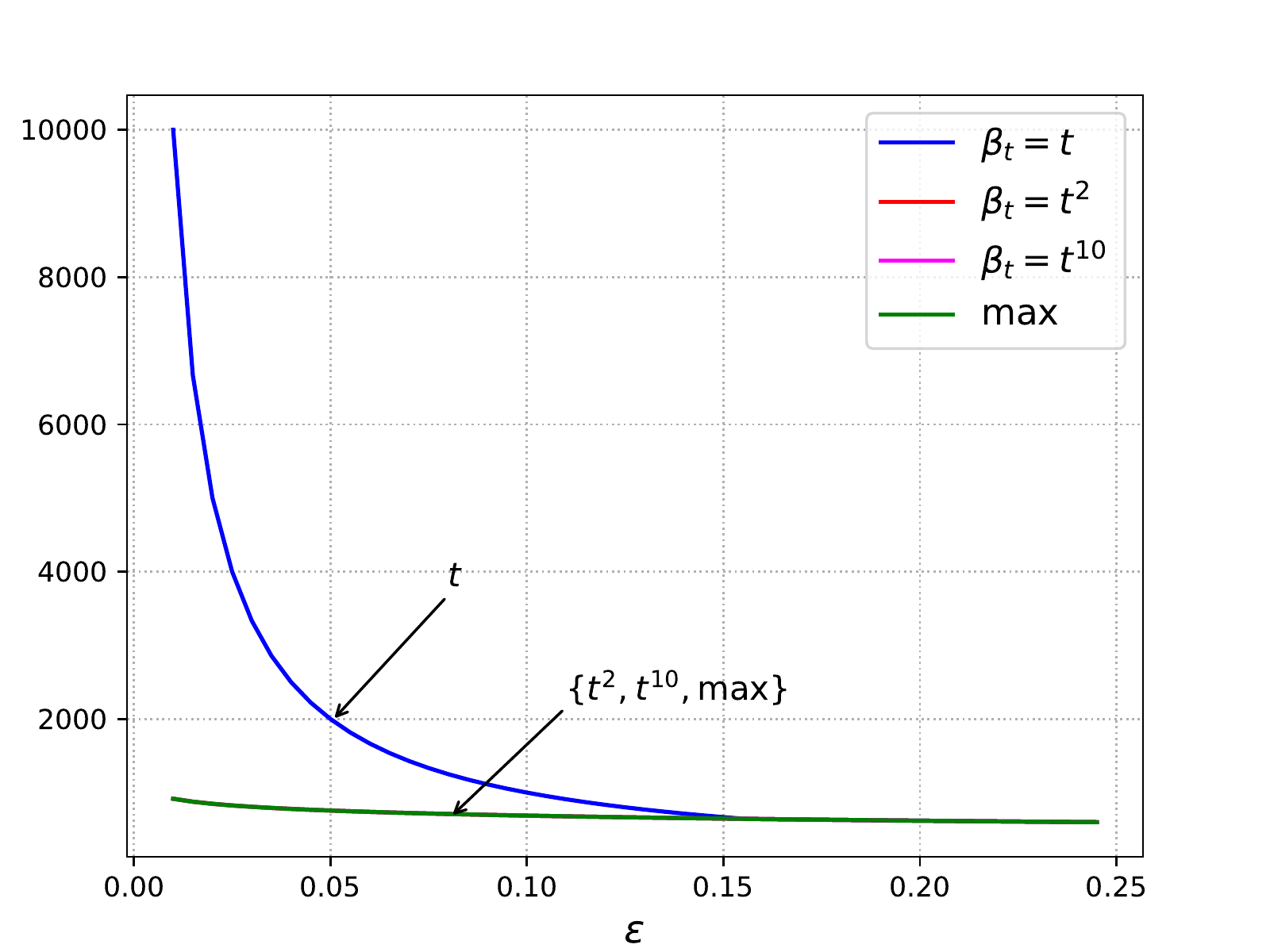}
		\end{minipage}
	}
	\caption{DBS value iteration in GridWorld.}
	\label{fig:vi_gridworld}
\end{figure}

The GridWorld consists of $10 \times 10$ grids, with the dark grids representing walls.
The agent starts at the upper left corner and aims to eat the apple at the bottom right corner upon receiving a reward of $+1$.
Otherwise, the reward is $0$.
An episode ends if the agent successfully eats the apple or a maximum number of steps $300$ is reached.
For this experiment, we consider the discount factor $\gamma = 0.9$.

The value loss of value iteration is shown in Figure \ref{fig:vi_gridworld}(b).
As expected, for fixed $\beta$, a larger value leads to a smaller loss.
We then zoom in on Figure \ref{fig:vi_gridworld}(b) to further illustrate the difference between fixed $\beta$ and dynamic $\beta_t$ in Figure \ref{fig:vi_gridworld}(c), which shows the value loss for the last episode in log scale.
For any fixed $\beta$, value iteration suffers from some loss which decreases as $\beta$ increases.
For dynamic $\beta_t$, the performance of $t^2$ and $t^3$ are the same and achieve the smallest loss in the domain game.
Results for the convergence rate is shown in Figure \ref{fig:vi_gridworld}(d).
For higher order $p$ of $\beta_t=t^p$, the convergence rate is faster.
We also see that the convergence rate of $t^2$ and $t^{10}$ is very close and matches the performance of the standard Bellman operator as discussed before.

From the above results, we find a convergent variant of the Boltzmann softmax operator with good convergence rate, which paves the path for its use in reinforcement learning algorithms with little knowledge little about the environment.

\subsection{Q-learning with DBS Updates}
\label{sec:q_theoy}
In this section, we show that the DBS operator can be applied in a model-free Q-learning algorithm.

\begin{algorithm}[!h]
	\caption{Q-learning with DBS updates}
	\label{algo:gen_q}
	\begin{algorithmic}[1]
		\STATE Initialize $Q(s,a), \forall s \in\mathcal{S}, a \in \mathcal{A}$ arbitrarily, and $Q(terminal \ state, \cdot)=0$ 
		\FOR {each episode $t=1, 2, ...$}
		\STATE Initialize $s$ 
		\FOR {each step of episode}
		\STATE $\triangleright${\textbf{action selection}}
		\STATE choose $a$ from $s$ using $\epsilon$-greedy policy
		\STATE take action $a$, observe $r, s'$ 
		\STATE $\triangleright${\textbf{value function estimation}}
		\STATE $V(s') = {\rm boltz}_{\beta_t} \left( Q(s', \cdot) \right)$
		\STATE $Q(s,a) \leftarrow Q(s,a) + \alpha_t \left[r + \gamma V(s') - Q(s,a) \right]$ 
		\STATE $s \leftarrow s'$
		\ENDFOR
		\ENDFOR
	\end{algorithmic}
\end{algorithm}

According to the DBS operator, we propose the DBS Q-learning algorithm (Algorithm \ref{algo:gen_q}).
Please note that the action selection policy is different from the Boltzmann distribution.

As seen in Theorem \ref{tm:dbs_con_rate}, a larger $p$ results in faster convergence rate in value iteration.
However, this is not the case in Q-learning, which differs from value iteration in that it knows little about the environment, and the agent has to learn from experience.
If $p$ is too large, it quickly approximates the max operator that favors commitment to current action-value function estimations.
This is because the max operator always greedily selects the maximum action-value function according to current estimation, which may not be accurate in the early stage of learning or in noisy environments.
As such, the max operator fails to consider other potential action-value functions.

\subsubsection{Theoretical analysis}
We get that DBS Q-learning converges to the optimal policy under the same additional condition as in DBS value iteration. The full proof is referred to the supplemental material.

Besides the convergence guarantee, we show that the Boltzmann softmax operator can mitigate the overestimation phenomenon of the max operator in Q-learning \cite{watkins1989learning} and the log-sum-exp operator in soft Q-learning \cite{haarnoja2017reinforcement}.

Let $X=\{ X_1, ..., X_M \}$ be a set of random variables, where the probability density function (PDF) and the mean of variable $X_i$ are denoted by $f_i$ and $\mu_i$ respectively. 
Please note that in value function estimation, the random variable $X_i$ corresponds to random values of action $i$ for a fixed state.
The goal of value function estimation is to estimate the maximum expected value $\mu^*(X)$,  and is defined as $\mu^*(X) = \max_i \mu_i = \max_i \int_{- \infty}^{+ \infty} x f_i(x) \text{d}x.$
However, the PDFs are unknown.
Thus, it is impossible to find $\mu^*(X)$ in an analytical way.
Alternatively, a set of samples $S = \{S_1, ..., S_M\}$ is given, where the subset $S_i$ contains independent samples of $X_i$.
The corresponding sample mean of $S_i$ is denoted by $\hat{\mu}_i$, which is an unbiased estimator of $\mu_i$. Let $\hat{F}_i$ denote the sample distribution of $\hat{\mu_i}$, $\hat{\mu}=(\hat{\mu}_1,...,\hat{\mu}_M)$, and $\hat{F}$ denote the joint distribution of $\hat{\mu}$.
The bias of any action-value summary operator $\bigotimes$ is defined as $\rm{Bias}(\hat{\mu}^*_{\bigotimes}) = \mathbb{E}_{\hat{\mu}\sim \hat{F}}[\bigotimes \hat{\mu}] - \mu^*(X),$ i.e., the difference between the expected estimated value by the operator over the sample distributions and the maximum expected value.

We now compare the bias for different common operators and we derive the following theorem, where the full proof is referred to the supplemental material.

\begin{theorem}
	Let $\hat{\mu}^*_{B_{\beta_t}}, \hat{\mu}^{*}_{\max}, \hat{\mu}^{*}_{L_{\beta}}$ denote the estimator with the DBS operator, the max operator, and the log-sum-exp operator, respectively.
	For any given set of $M$ random variables, we have
	$\forall t, \ \forall \beta,$
	\begin{equation}
	\rm{Bias}(\hat{\mu}^*_{B_{\beta_t}}) \leq \rm{Bias}(\hat{\mu}^{*}_{\max}) \leq \rm{Bias}(\hat{\mu}^{*}_{L_{\beta}}).
	\end{equation}
	\label{thm:overestimation}
\end{theorem}

In Theorem \ref{thm:overestimation}, we show that although the log-sum-exp operator \cite{haarnoja2017reinforcement} is able to encourage exploration because its objective is an entropy-regularized form of the original objective, it may worsen the overestimation phenomenon.
In addition, the optimal value function induced by the log-sum-exp operator is biased from the optimal value function of the original MDP \cite{dai2018sbeed}.  
In contrast, the DBS operator ensures convergence to the optimal value function as well as reduction of overestimation.
 
\subsubsection{Empirical Results}
We now evaluate the performance of DBS Q-learning in the same GridWorld environment.
Figure \ref{fig:grid_world_q} demonstrates the number of steps the agent spent until eating the apple in each episode, and a fewer number of steps the agent takes corresponds to a better performance.

\begin{figure}[!h]
\centering 
 \includegraphics[scale=0.4]{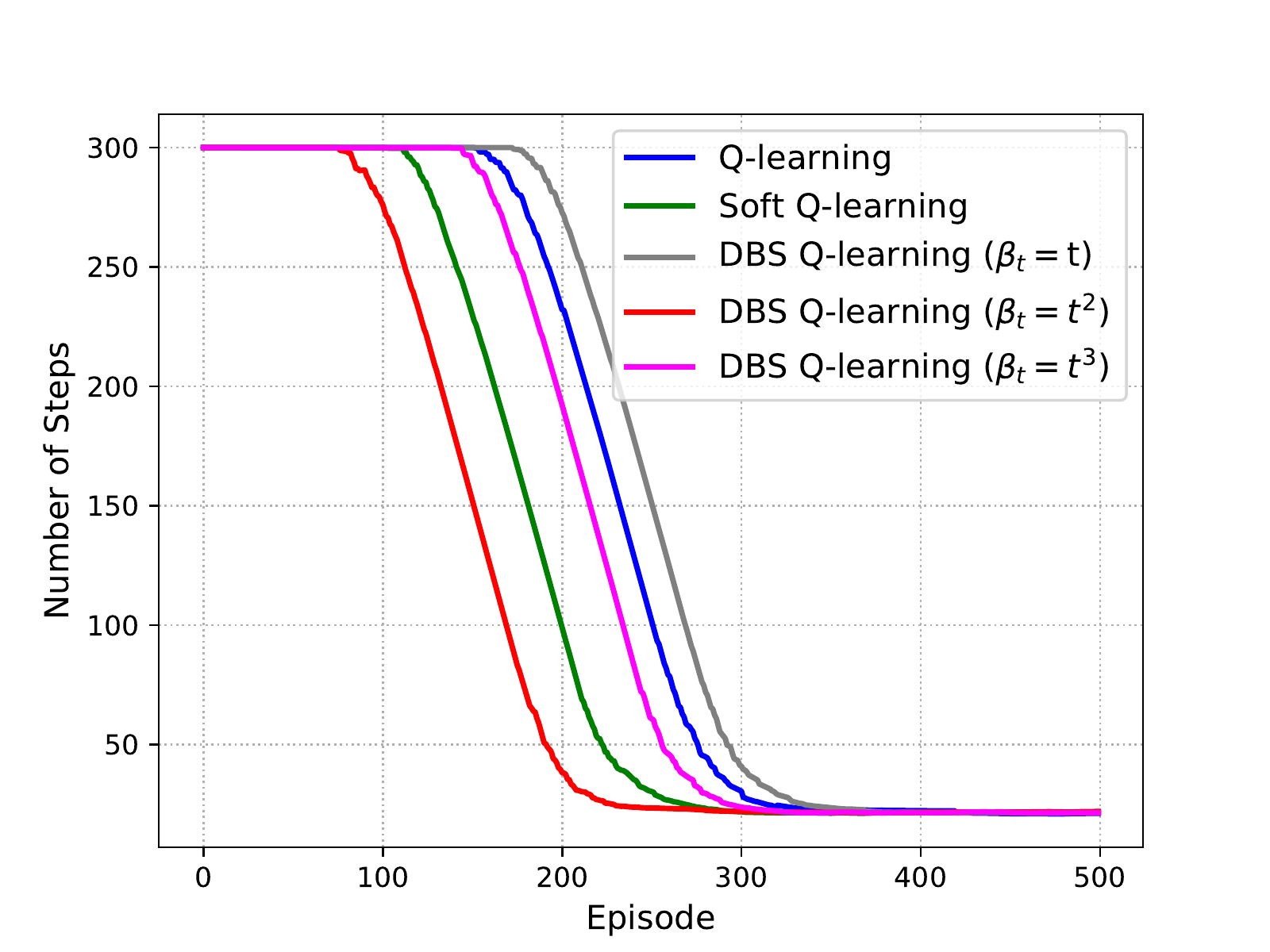} 
 \caption{Performance comparison of DBS Q-learning, Soft Q-learning, and Q-learning in GridWorld.}
\label{fig:grid_world_q}
\end{figure}

For DBS Q-learning, we apply the power function $\beta_t = t^p$ with $p$ denoting the order.
As shown, DBS Q-learning with the quadratic function achieves the best performance. 
Note that when $p=1$, it performs worse than Q-learning in this simple game, which corresponds to our results in value iteration (Figure \ref{fig:vi_gridworld}) as $p=1$ leads to an unnegligible value loss.
When the power $p$ of $\beta_t=t^p$ increases further, it performs closer to Q-learning. 

Soft Q-learning \cite{haarnoja2017reinforcement} uses the log-sum-exp operator, where the parameter is chosen with the best performance for comparison. 
Readers please refer to the supplemental material for full results with different parameters.  
In Figure \ref{fig:grid_world_q}, soft Q-learning performs better than Q-learning as it encourages exploration according to its entropy-regularized objective.
However, it underperforms DBS Q-learning $(\beta_t = t^2)$ as DBS Q-learning can guarantee convergence to the optimal value function and can eliminate the overestimation phenomenon.
Thus, we choose $p=2$ in the following Atari experiments.

\section{The DBS-DQN Algorithm}
In this section, we show that the DBS operator can further be applied to problems with high dimensional state space and action space. 

The DBS-DQN algorithm is shown in Algorithm 2. 
We compute the parameter of the DBS operator by applying the power function $\beta_t(c)=c \cdot t^2$ as the quadratic function performs the best in our previous analysis.
Here, $c$ denote the coefficient, and contributes to controlling the speed of the increase of $\beta_t(c)$.
In many problems, it is critical to choose the hyper-parameter $c$.
In order to make the algorithm more practical in problems with high-dimensional state spaces, we propose to learn to adjust $c$ in DBS-DQN by the meta gradient-based optimization technique based on \cite{xu2018meta}.

The main idea of gradient-based optimization technique is summarized below, which follows the online cross-validation principle \cite{sutton1992adapting}.
Given current experience $\tau=(s,a,r,s_{next})$, the parameter $\theta$ of the function approximator is updated according to 
\begin{equation}
\theta' = \theta - \alpha \frac{\partial J(\tau, \theta, c)}{\partial \theta},
\label{eq:theta_update}
\end{equation}
where $\alpha$ denotes the learning rate, and the loss of the neural network is
\begin{equation}
\begin{split}
&J(\tau, \theta, c) = \frac{1}{2} \left[ V(\tau, c;\theta^-) -Q(s,a;\theta)\right]^2, \\
&V(\tau,c;\theta^-) = r + \gamma {\rm boltz}_{\beta_t(c)} \left( Q(s_{next}, \cdot; \theta^-) \right),
\end{split}
\end{equation}
with $\theta^-$ denoting the parameter of the target network.
The corresponding gradient of $J(\tau, \theta, c)$ over $\theta$ is 
\begin{equation}
\begin{split}
\frac{\partial J(\tau, \theta, c)}{\partial \theta} =  -\big[ & r + \gamma \text{boltz}_{\beta_t(c)} (Q(s_{next}, \cdot; \theta^-)) \\
&  - Q(s, a; \theta) \big] \frac{\partial Q(s,a;\theta)}{\partial \theta}.
\end{split}
\end{equation}
Then, the coefficient $c$ is updated based on the subsequent experience $\tau'=(s',a,r',s'_{netx})$ according to the gradient of the squared error $J(\tau', \theta', \bar{c})$ between the value function approximator $Q(s'_{next},a';\theta')$ and the target value function $V(\tau', \bar{c};\theta^-)$, where $\bar{c}$ is the reference value.
The gradient is computed according to the chain rule in Eq. (\ref{eq:chain}).
\begin{equation}
\frac{\partial J'(\tau', \theta', \bar{c})}{\partial c} = \underbrace{  \frac{\partial J'(\tau', \theta', \bar{c})}{\partial \theta'}  }_{A} \underbrace{  \frac{\text{d} \theta'}{\text{d} c}  }_{B}.
\label{eq:chain}
\end{equation}
For the term (B), according to Eq. (\ref{eq:theta_update}), we have
\begin{equation}
\frac{\text{d} \theta'}{\text{d} c} = \alpha \gamma  \frac{\partial {\rm{boltz}}_{\beta_t(\bar{c})} (Q(s'_{next}, \cdot; \theta^-))}{\partial c}
\frac{\partial Q(s,a;\theta)}{\partial \theta}.
\end{equation}
Then, the update of $c$ is 
\begin{equation}
c' = c - \beta \frac{\partial J'(\tau', \theta', \bar{c})}{\partial c},
\end{equation}
with $\eta$ denoting the learning rate.

Note that it can be hard to choose an appropriate static value of sensitive parameter $\beta$.
Therefore, it requires rigorous tuning of the task-specific fixed parameter $\beta$ in different games in \cite{song2018revisiting}, which may limit its efficiency and applicability \cite{haarnoja2018learning}.
In contrast, the DBS operator is effective and efficient as it does not require tuning.

\begin{algorithm}[!h]
	\begin{algorithmic}[1]
		\caption{DBS Deep Q-Network}
		\STATE initialize experience replay buffer $\mathcal{B}$
		\STATE initialize Q-function and target Q-function with random weights $\theta$ and $\theta^-$
		\STATE initialize the coefficient $c$ of the parameter $\beta_t$ of the DBS operator
		\FOR {episode = 1, ..., M}
			\STATE initialize state $s_1$
			\FOR {step = 1, ..., T}
				\STATE choose $a_t$ from $s_t$ using $\epsilon$-greedy policy
				\STATE execute $a_t$, observe reward $r_t$, and next state $s_{t+1}$
				\STATE store experience $(s_t, a_t, r_t, s_{t+1})$ in $\mathcal{B}$
				\STATE calculate $\beta_t(c) = c \cdot t^2$
				\STATE sample random minibatch of experiences $(s_j, a_j, r_j, s_{j+1})$ from $\mathcal{B}$
				\IF {$s_{j+1}$ is terminal state}
					\STATE set $y_j = r_j$
				\ELSE
					\STATE set $y_j = r_j + \gamma {\rm{boltz}}_{\beta_t} \left( \hat{Q}(s_{j+1}, \cdot; \theta^-) \right)$
				\ENDIF
				\STATE perform a gradient descent step on $\left( y_j - Q(s_j, a_j; \theta) \right)^2$ w.r.t. $\theta$
				\STATE update $c$ according to the gradient-based optimization technique
				\STATE reset $\hat{Q} = Q$ every $C$ steps
			\ENDFOR
		\ENDFOR
		\label{algo:dbs_dqn}
	\end{algorithmic}
\end{algorithm}

\subsection{Experimental Setup}
We evaluate the DBS-DQN algorithm on 49 Atari video games from the Arcade Learning Environment \cite{bellemare2013arcade}, a standard challenging benchmark for deep reinforcement learning algorithms, by comparing it with DQN.
For fair comparison, we use the same setup of network architectures and hyper-parameters as in \cite{mnih2015human} for both DQN and DBS-DQN.
Our evaluation procedure is $30$ no-op evaluation which is identical to \cite{mnih2015human}, where the agent performs a random number (up to $30$) of ``do nothing'' actions in the beginning of an episode.
See the supplemental material for full implementation details.

\subsection{Effect of the Coefficient $c$}
The coefficient $c$ contributes to the speed and degree of the adjustment of $\beta_t$, and we propose to learn $c$ by the gradient-based optimization technique \cite{xu2018meta}.
It is also interesting to study the effect of the coefficient $c$ by choosing a fixed parameter, and we train DBS-DQN with differnt fixed paramters $c$ for 25M steps (which is enough for comparing the performance).
As shown in Figure \ref{fig:seaquest}, DBS-DQN with all of the different fixed parameters $c$ outperform DQN, and DBS-DQN achieves the best performance compared with all choices of $c$.
\begin{figure}[!h]
\centering
\includegraphics[scale=0.35]{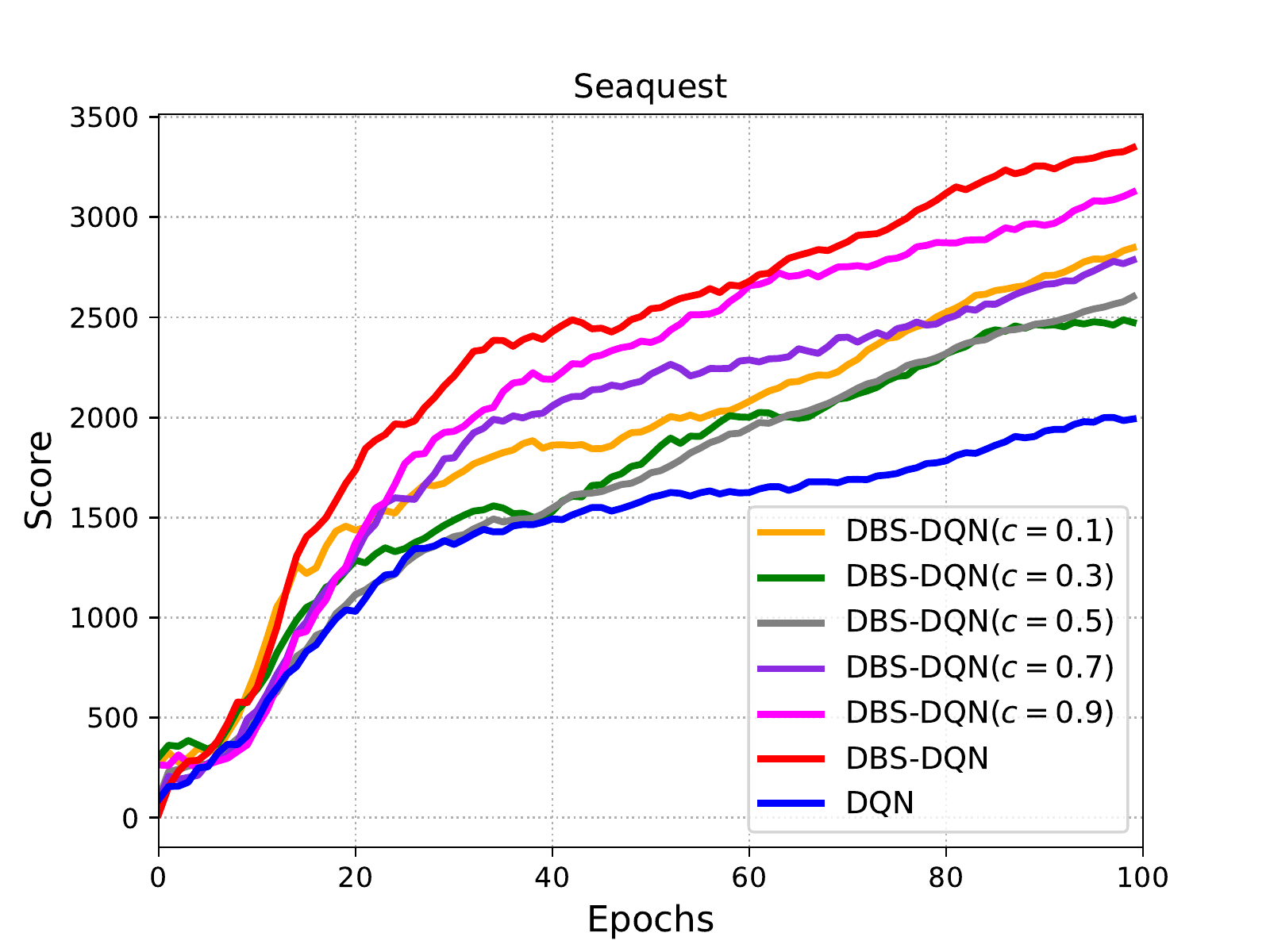}
\caption{Effect of the coefficient $c$ in the game Seaquest.}
\label{fig:seaquest}
\end{figure}

\subsection{Performance Comparison}
We evaluate the DBS-DQN algorithm on 49 Atari video games from the Arcade Learning Environment (ALE) \cite{bellemare2013arcade}, by comparing it with DQN. 
For each game, we train each algorithm for $50$M steps for 3 independent runs to evaluate the performance.
Table \ref{table:atari_mean_median} shows the summary of the median in human normalized score \cite{van2016deep} defined as:
\begin{equation}
\frac{ \rm {score_{agent} - score_{random}}  }{  \rm{score_{human} - score_{random}}  } \times 100\%,
\end{equation}
where human score and random score are taken from \cite{wang2015dueling}.
As shown in Table \ref{table:atari_mean_median}, DBS-DQN significantly outperforms DQN in terms the median of the human normalized score, and surpasses human level.
In all, DBS-DQN exceeds the performance of DQN in 40 out of 49 Atari games, and Figure \ref{fig: atari_learning_curve} shows the learning curves (moving averaged). 

\begin{table}[!h]
	\centering
	\caption{Summary of Atari games.}
	\begin{tabular}{lc}
		\toprule
		{\bf{Algorithm}} & {\bf{Median}} \\
		\midrule
		{\bf{DQN}} & 84.72\% \\
		{\bf{DBS-DQN}} & 104.49\% \\
		{\bf{DBS-DQN (fine-tuned $c$)}} & 103.95\% \\
		\bottomrule
	\end{tabular}
	\label{table:atari_mean_median}
\end{table}

\begin{figure}[!h]
	\centering
	\subfigure[Frostbite
	]{
		\begin{minipage}[t]{0.5\linewidth}
			\centering
			\includegraphics[scale=0.33]{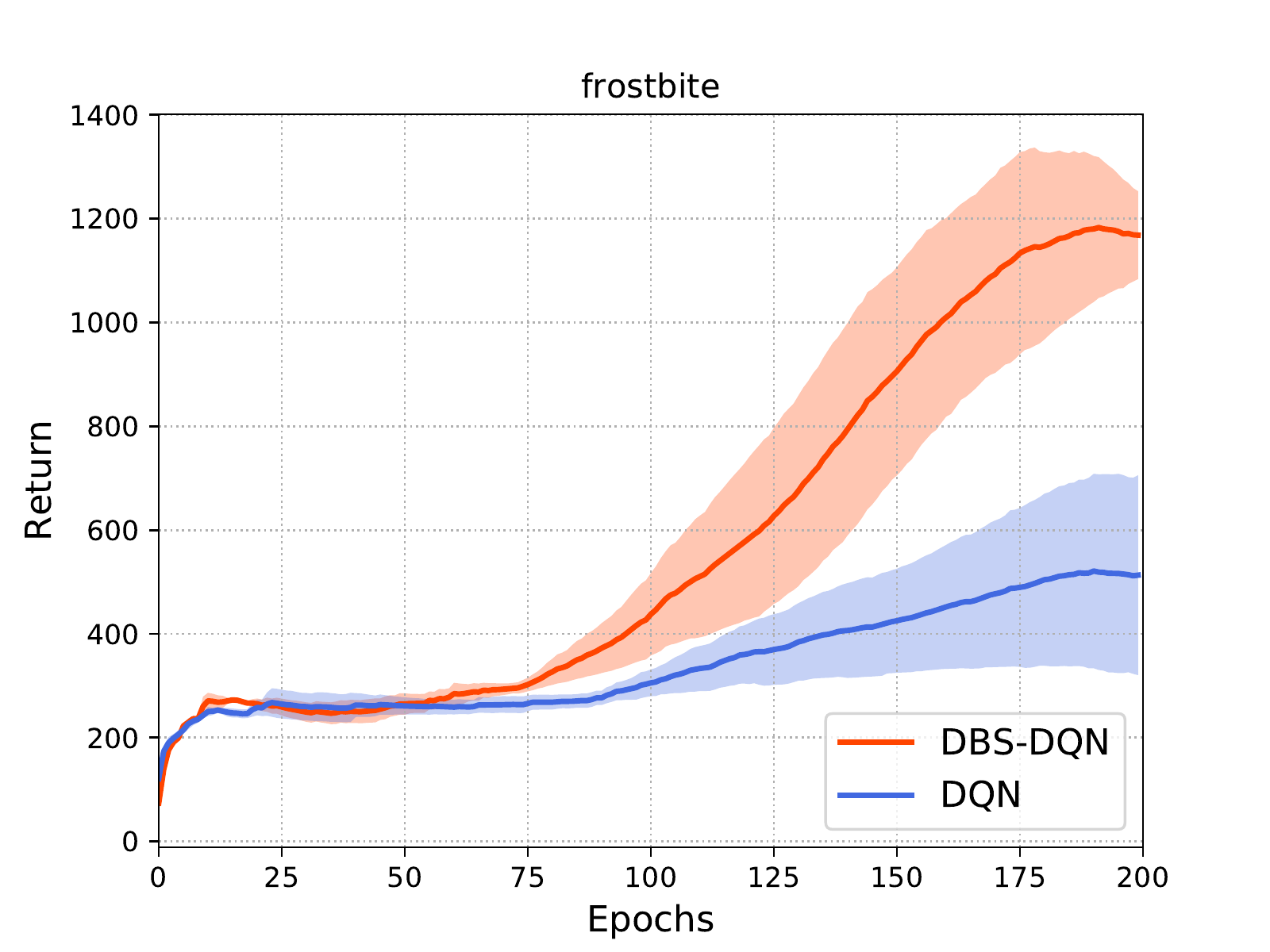}
		\end{minipage}
	}%
	\subfigure[IceHockey
	]{
		\begin{minipage}[t]{0.5\linewidth}
			\centering
			\includegraphics[scale=0.33]{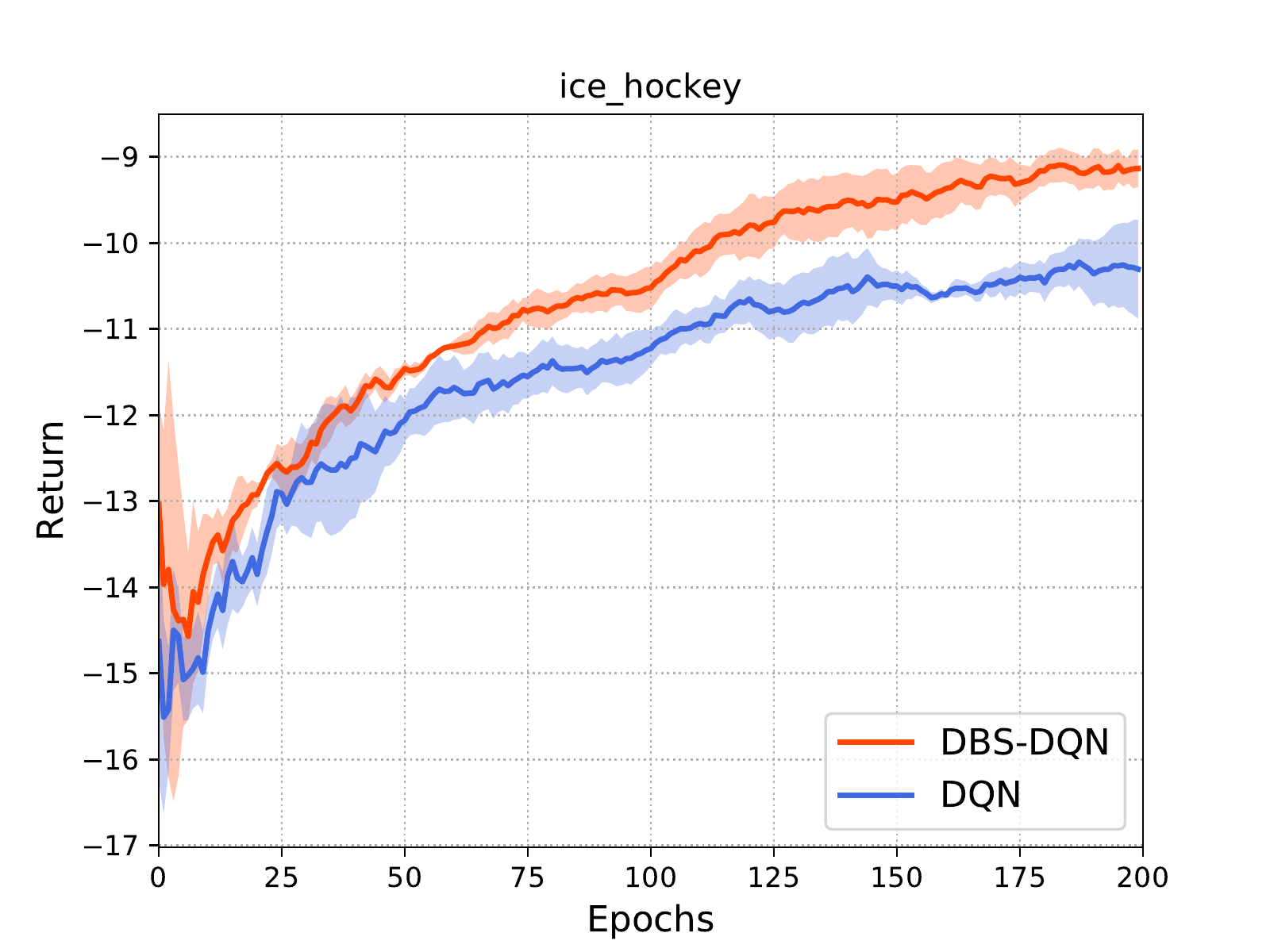}
		\end{minipage}
	}
	\subfigure[Riverraid
	]{
		\begin{minipage}[t]{0.5\linewidth}
			\centering
			\includegraphics[scale=0.33]{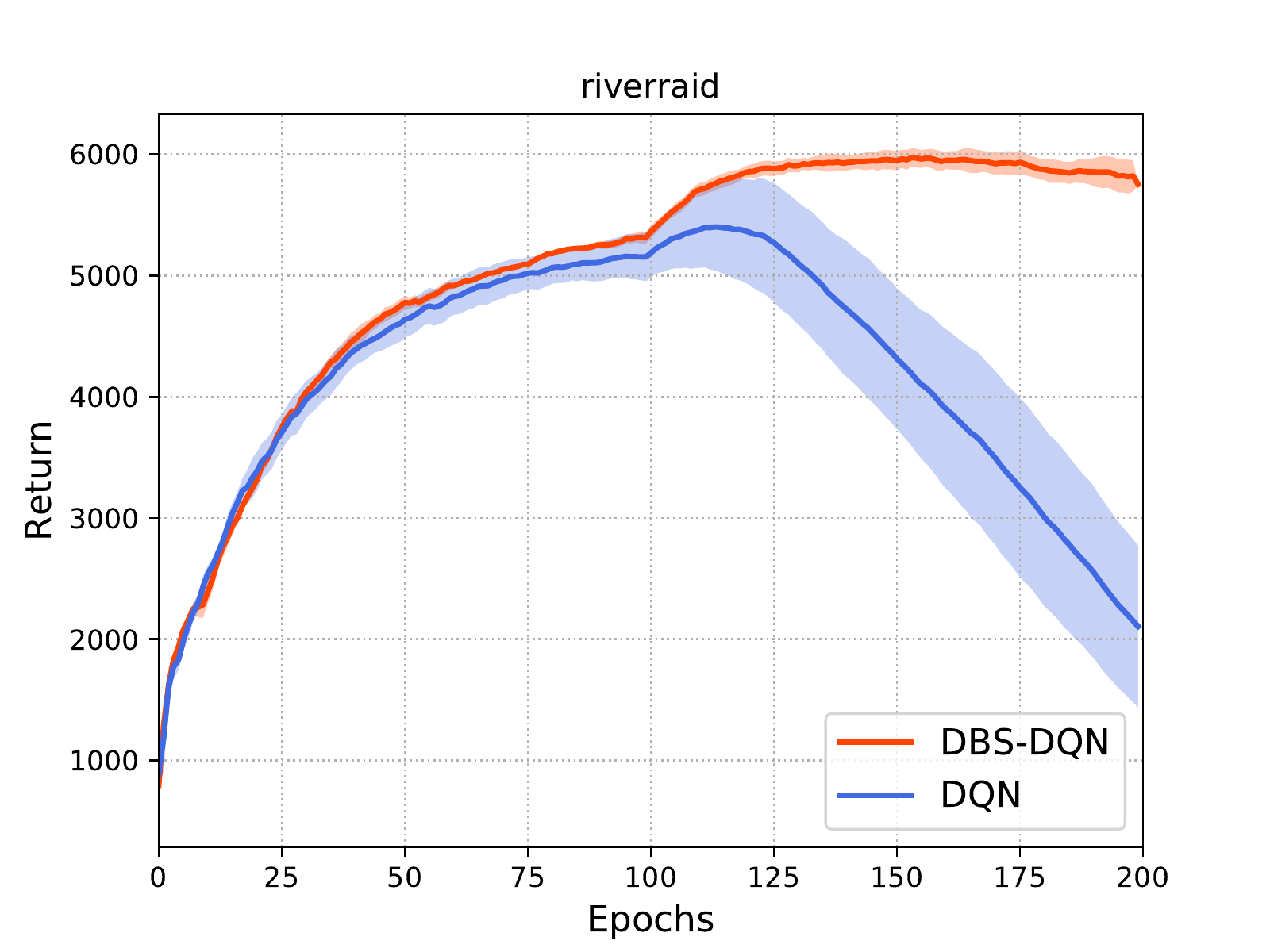}
		\end{minipage}
	}%
	\subfigure[RoadRunner
	]{
		\begin{minipage}[t]{0.5\linewidth} 
			\centering
			\includegraphics[scale=0.33]{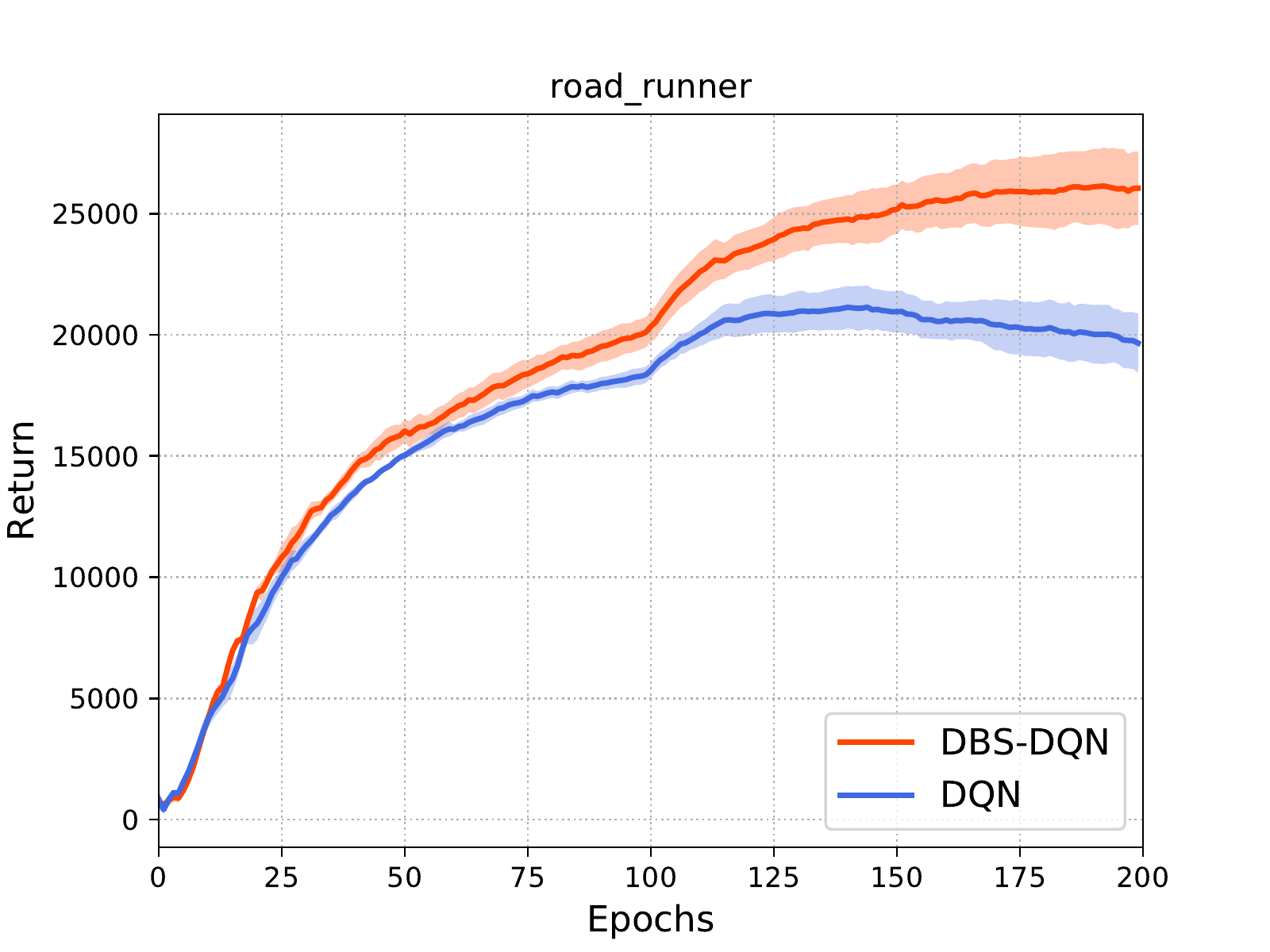}
		\end{minipage}
	}
	\subfigure[Seaquest
	]{
		\begin{minipage}[t]{0.5\linewidth}
			\centering
			\includegraphics[scale=0.33]{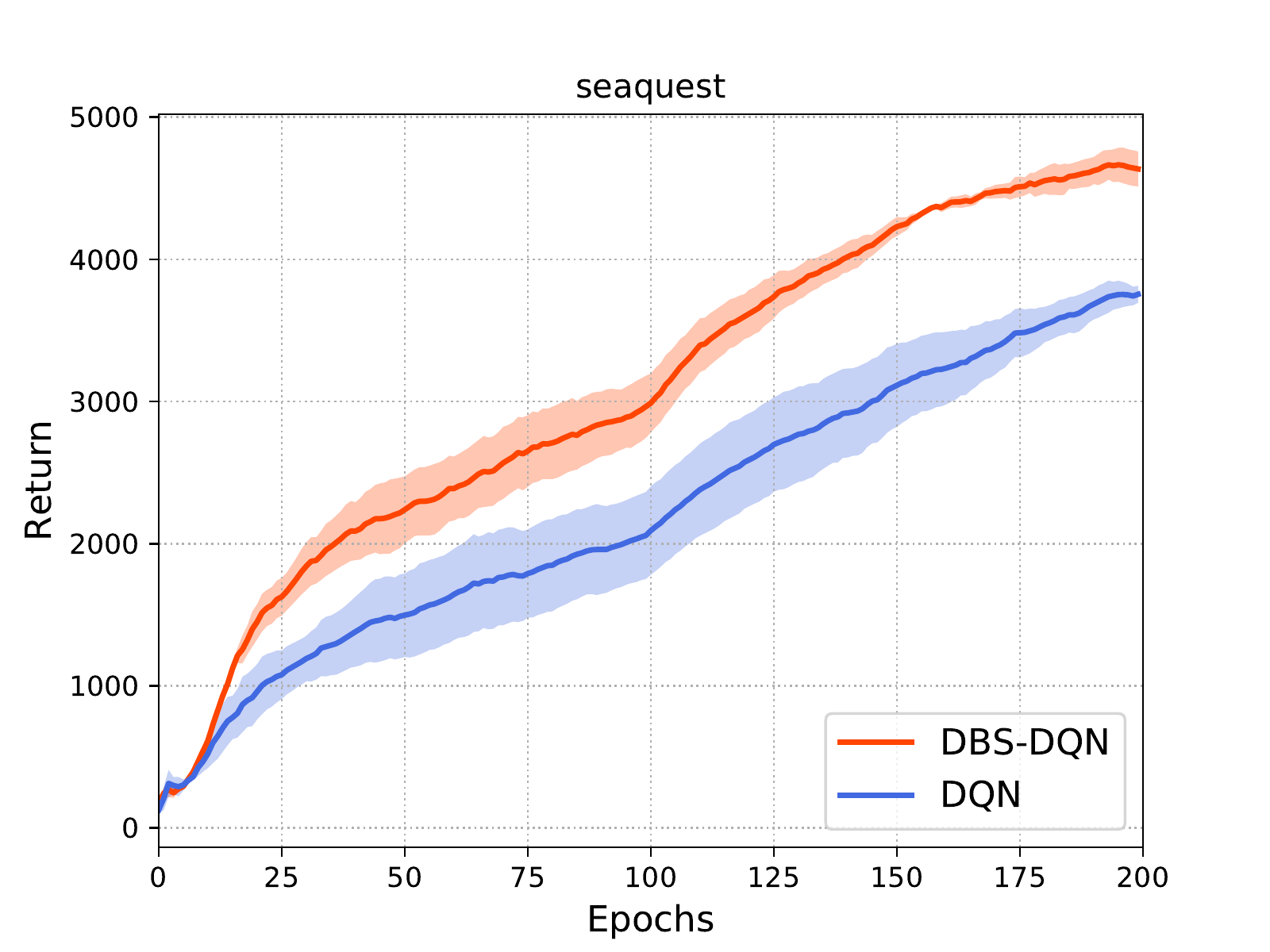}
		\end{minipage}
	}%
	\subfigure[Zaxxon
	]{
		\begin{minipage}[t]{0.5\linewidth}
			\centering
			\includegraphics[scale=0.33]{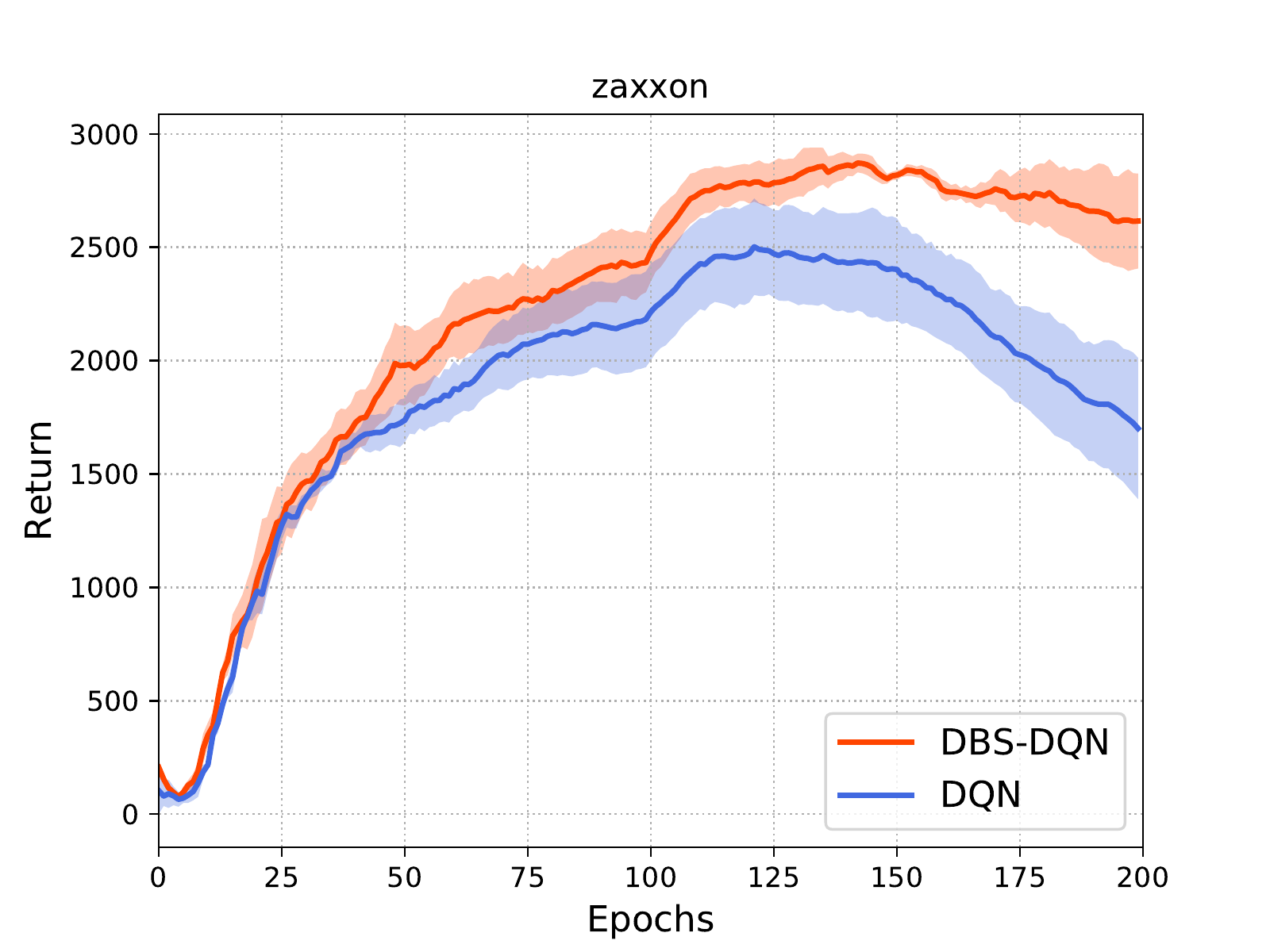}
		\end{minipage}
	}
	\caption{Learning curves in Atari games.}
	\label{fig: atari_learning_curve}
\end{figure}

To demonstrate the effectiveness and efficiency of DBS-DQN, we compare it with its variant with fine-tuned fixed coefficient $c$ in $\beta_t(c)$, i.e., without graident-based optimization, in each game.
From Table \ref{table:atari_mean_median}, DBS-DQN exceeds the performance of DBS-DQN (fine-tuned $c$), which shows that it is effective and efficient as it performs well in most Atari games which does not require tuning.
It is also worth noting that DBS-DQN (fine-tuned $c$) also achieves fairly good performance in term of the median and beats DQN in 33 out of 49 Atari games, which further illustrate the strength of our proposed DBS updates without gradient-based optimization of $c$.
Full scores of comparison is referred to the supplemental material.

\section{Related Work}
The Boltzmann softmax distribution is widely used in reinforcement learning \cite{littman1996reinforcement,sutton1998introduction,azar2012dynamic,song2018revisiting}.
Singh et al. \cite{singh2000convergence} studied convergence of on-policy algorithm Sarsa, 
where they considered a dynamic scheduling of the parameter in softmax action selection strategy.
However, the state-dependent parameter is impractical in complex problems, e.g., Atari.
Our work differs from theirs as our DBS operator is state-independent, which can be readily scaled to complex problems with high-dimensional state space.
Recently, \cite{song2018revisiting} also studied the error bound of the Boltzmann softmax operator and its application in DQNs.
In contrast, we propose the DBS operator which rectifies the convergence issue of softmax, where we provide a more general analysis of the convergence property.
A notable difference in the theoretical aspect is that we achieve a tighter error bound for softmax in general cases, and we investigate the convergence rate of the DBS operator. 
Besides the guarantee of Bellman optimality, the DBS operator is efficient as it does not require hyper-parameter tuning.
Note that it can be hard to choose an appropriate static value of $\beta$ in \cite{song2018revisiting}, which is game-specific and can result in different performance. 

A number of studies have studied the use of alternative operators, most of which satisfy the non-expansion property \cite{haarnoja2017reinforcement}.
\cite{haarnoja2017reinforcement} utilized the log-sum-exp operator, which enables better exploration and learns deep energy-based policies.
The connection between our proposed DBS operator and the log-sum-exp operator is discussed above.
\cite{bellemare2016increasing} proposed a family of operators which are not necessarily non-expansions, but still preserve optimality while being gap-increasing. However, such conditions are still not satisfied for the Boltzmann softmax operator. 

\section{Conclusion}
We propose the dynamic Boltzmann softamax (DBS) operator in value function estimation with a time-varying, state-independent parameter.
The DBS operator has good convergence guarantee in the setting of planning and learning, which rectifies the convergence issue of the Boltzmann softmax operator.
Results validate the effectiveness of the DBS-DQN algorithm in a suite of Atari games. 
For future work, it is worth studying the sample complexity of our proposed DBS Q-learning algorithm. It is also promising to apply the DBS operator to other state-of-the-art DQN-based algorithms, such as Rainbow \cite{hessel2017rainbow}. 

\clearpage
\appendix

\renewcommand\thesection{\Alph{section}} 

\section{Convergence of DBS Value Iteration}\label{app:dbs_vi_con}
\textbf{Proposition 1}
\emph{
	\begin{equation}
	L_{\beta}({\bf{X}}) - \text{boltz}_{\beta} ({\bf{X}}) = \frac{1}{\beta} \sum_{i=1}^n -p_i \log(p_i) \leq \frac{\log(n)}{\beta},
	\end{equation}
	where $p_i = \frac{e^{\beta x_i}}{\sum_{j=1}^n e^{\beta x_j}}$ denotes the weights of the Boltzmann distribution, $L_{\beta}({\bf{X}})$ denotes the log-sum-exp function $L_{\beta}({\bf{X}}) = \frac{1}{\beta} \log(\sum_{i=1}^n e^{\beta x_i})$, and $\text{boltz}_{\beta}({\bf{X}})$ denotes the Boltzmann softmax function $\text{boltz}_{\beta}({\bf{X}})=  \frac{\sum_{i=1}^n e^{\beta x_i}x_i}{\sum_{j=1}^n e^{\beta x_j}}$.
}

\begin{proof}
	\begin{align}
	&\frac{1}{\beta} \sum_{i=1}^n -p_i \log(p_i) \\
	= &\frac{1}{\beta} \sum_{i=1}^n \left( -\frac{e^{\beta x_i}}{\sum_{j=1}^n e^{\beta x_j}} \log \left(\frac{e^{\beta x_i}}{\sum_{j=1}^n e^{\beta x_j}} \right) \right) \\
	= &\frac{1}{\beta} \sum_{i=1}^n \left( -\frac{e^{\beta x_i}}{\sum_{j=1}^n e^{\beta x_j}} \left( \beta x_i - \log \left( \sum_{j=1}^n e^{\beta x_j} \right) \right) \right) \\
	= &-\sum_{i=1}^n \frac{e^{\beta x_i} x_i}{\sum_{j=1}^n e^{\beta x_j}} + \frac{1}{\beta} \log \left( \sum_{j=1}^n e^{\beta x_j} \right) \frac{\sum_{i=1}^n e^{\beta x_i}}{\sum_{j=1}^n e^{\beta x_j}} \\
	= &-\text{boltz}_{\beta} ({\bf{X}}) + L_{\beta}({\bf{X}})
	\end{align}
	
	Thus, we obtain
	\begin{equation}
	L_{\beta}({\bf{X}}) - \text{boltz}_{\beta} ({\bf{X}}) = \frac{1}{\beta} \sum_{i=1}^n -p_i \log(p_i),
	\end{equation}
	where $\frac{1}{\beta} \sum_{i=1}^n -p_i \log(p_i)$ is the entropy of the Boltzmann distribution.
	
	It is easy to check that the maximum entropy is achieved when $p_i = \frac{1}{n}$, where the entropy equals to $\log(n)$.
	
\end{proof}

\textbf{Theorem 1 (Convergence of value iteration with the DBS operator)} \emph{
	For any dynamic Boltzmann softmax operator $ \beta_t $,
	if $\beta_t \to \infty$,
	$V_t$ converges to $V^*$, where $V_t$ and $V^*$ denote the value function after $t$ iterations and the optimal value function.
}

\begin{proof}
	By the definition of $\mathcal{T}_{\beta_t}$ and $\mathcal{T}_m$, we have
	\begin{equation}
	\begin{split}
	& ||(\mathcal{T}_{\beta_t} V_1) - (\mathcal{T}_m V_2) ||_{\infty} \\
	\leq & \underbrace{|| (\mathcal{T}_{\beta_t} V_1) - (\mathcal{T}_m V_1) ||_{\infty}}_{(A)} + \underbrace{|| (\mathcal{T}_m V_1) - (\mathcal{T}_m V_2) ||_{\infty}}_{(B)} \\
	\end{split}
	\label{th1_1}
	\end{equation}
	
	For the term $(A)$, we have
	\begin{align}
	& ||(\mathcal{T}_{\beta_t} V_1) - (\mathcal{T}_m V_1)||_{\infty} \\
	=& \max_s | {\rm{boltz}}_{\beta_t} (Q_1(s,\cdot)) - \max_a (Q_1(s,a)) | \\
	\leq & \max_s | {\rm{boltz}}_{\beta_t}(Q_1(s, \cdot)) - L_{\beta_t} (Q_1(s,a)) |  \\
	\leq & \frac{\log(|A|)}{\beta_t}, 
	\label{th1_2}
	\end{align}
	where Ineq. (\ref{th1_2}) is derived from Proposition 1.
	
	For the term $(B)$, we have
	\begin{align}
	& || (\mathcal{T}_m V_1) - (\mathcal{T}_m V_2) ||_{\infty} \\
	=& \max_s | \max_{a_1} (Q_1(s, a_1)) - \max_{a_2} (Q_2(s, a_2)) | \\
	\leq & \max_s \max_a | Q_1(s,a) - Q_2(s,a) | \\
	\leq & \max_s \max_a \gamma \sum_{s'} p(s'|s,a) |V_1(s') - V_2(s') | \\
	\leq & \gamma ||V_1 - V_2||
	\label{th1_3}
	\end{align}
	
	Combing (\ref{th1_1}), (\ref{th1_2}), and (\ref{th1_3}), we have
	\begin{equation}
	|| (\mathcal{T}_{\beta_t} V_1) - (\mathcal{T}_m V_2) ||_{\infty} \leq \gamma ||V_1 - V_2||_{\infty} + \frac{\log(|A|)}{\beta_t} \\
	\end{equation}
	
	As the $\max$ operator  is a contraction mapping, then from Banach fixed-point theorem we have
	$\mathcal{T}_m V^* = V^*$
	
	By definition we have
	\begin{align}
	& ||V_t - V^*||_{\infty} \\
	= &  || (\mathcal{T}_{\beta_t} ... \mathcal{T}_{\beta_1}) V_0 - (\mathcal{T}_m ... \mathcal{T}_m) V^*||_{\infty} \\
	\leq & \gamma || (\mathcal{T}_{\beta_{t-1}} ... \mathcal{T}_{\beta_1}) V_0 - (\mathcal{T}_m ... \mathcal{T}_m) V^*||_{\infty} + \frac{\log(|A|)}{\beta_t} \\
	\leq & ... \\
	\leq & \gamma^t ||V_0 - V^*||_{\infty} + \log(|A|) \sum_{k=1}^t \frac{\gamma^{t-k}}{\beta_k}
	\label{eq:vi_con_basic}
	\end{align}
	
	We prove that $\lim_{t \to \infty}\sum_{k=1}^t \frac{\gamma^{t-k}}{\beta_k}=0$.
		
	Since $\lim_{k \to \infty} \frac{1}{\beta_k} = 0$, we have that
	$\forall \epsilon_1 > 0, \exists K(\epsilon_1) > 0, \text{ such that } \forall k > K(\epsilon_1), |\frac{1}{\beta_k}| < \epsilon_1.$
	Thus,
	\begin{align}
	\label{ieq: lim}
	& \sum_{k=1}^t \frac{\gamma^{t - k}}{\beta_k} \\
	=& \sum_{k=1}^{K(\epsilon_1)} \frac{\gamma^{t - k}}{\beta_k} + \sum_{k=K(\epsilon_1)+1}^t \frac{\gamma^{t - k}}{\beta_k} \\
	\leq & \frac{1}{\min_{k \leq t} \beta_k} \sum_{k=1}^{K(\epsilon_1)} \gamma^{t - k} + \epsilon_1 \sum_{k=K(\epsilon_1)+1}^t \gamma^{t-k} \\
	=& \frac{1}{\min_{k \leq t} \beta_k} \frac{\gamma^{t-K(\epsilon_1)} (1 - \gamma^{K(\epsilon_1)})}{1 - \gamma} + \epsilon_1 \frac{1 (1 - \gamma^{t-K(\epsilon_1)})}{1 - \gamma} \\
	\leq & \frac{1}{1 - \gamma} \big( \frac{\gamma^{t-K(\epsilon_1)}}{\min_{k \leq t} \beta_k} + \epsilon_1 \big)
	\end{align}
	
	If $t > \frac{\log ((\epsilon_2(1 - \gamma) - \epsilon_1) \min_{k \leq t} \beta_k)}{\log \gamma} + K(\epsilon_1)$ and $\epsilon_1 < \epsilon_2(1 - \gamma)$, then 
	\begin{equation}
	\sum_{k=1}^t \frac{\gamma^{t - k}}{\beta_k} < \epsilon_2.
	\end{equation}
	
	So we obtain that
	$\forall \epsilon_2 > 0, \exists T > 0$, such that
	\begin{equation}
	\forall t > T, |\sum_{k=1}^t \frac{\gamma^{t-k}}{\beta_k}| < \epsilon_2.
	\end{equation}
	
	Thus, $\lim_{t \to \infty} \sum_{k=1}^t \frac{\gamma^{t - k}}{\beta_k} = 0.$
	
	Taking the limit of the right side of the inequality (\ref{eq:vi_con_basic}), we have that
	\begin{equation}
	\lim_{t \to \infty} \left[ \gamma^t ||V_1 - V^*||_{\infty} + \log(|A|) \sum_{k=1}^t \frac{\gamma^{t-k}}{\beta_k} \right] = 0
	\end{equation}
	
	Finally, we obtain
	\begin{equation}
	\lim_{t \to \infty} ||V_t - V^*||_{\infty} = 0.
	\end{equation}
	
\end{proof}

\section{Convergence Rate of DBS Value Iteration}\label{app:con_bound} 
\textbf{Theorem 2 (Convergence rate of value iteration with the DBS operator)} \emph{
	For any power series $\beta_t = t^p (p > 0)$, let $V_0$ be an arbitrary initial value function such that $||V_0||_{\infty}\leq \frac{R}{1-\gamma}$, where $R=\max_{s,a}|r(s,a)|$, 
	we have that for any non-negative $\epsilon < 1/4$, after
	$\max \{ O\big( \frac{\log(\frac{1}{\epsilon})  + \log(\frac{1}{1 - \gamma}) + \log (R)}{\log(\frac{1}{\gamma})}), O\big({( \frac{1}{(1 - \gamma) \epsilon})}^{\frac{1}{p}}\big) \}$
	steps, the error $||V_t - V^*||_{\infty} \leq \epsilon$.
}

\begin{proof}
	\begin{align}
	\sum_{k=1}^t \frac{\gamma^{t-k}}{k^p}
	&= \gamma^t \big[ \sum_{k=1}^{\infty} \frac{\gamma^{-1}}{k^p} - \sum_{k=t+1}^{\infty} \frac{\gamma^{-1}}{k^p} \big] \\
	&= \gamma^t \big[ \underbrace{{\rm Li}_p(\gamma^{-1})}_{\rm Polylogarithm} - \gamma^{-(t+1)} \underbrace{\Phi(\gamma^{-1}, p, t+1)}_{\rm Lerch \ transcendent} \big]
	\label{bound_eq_1}
	\end{align}
	
	By \cite{ferreira2004asymptotic}, we have
	\begin{align}
	{\rm Eq.} \ (\ref{bound_eq_1}) &= \Theta \left( \gamma^t \frac{\gamma^{-(t+1)}}{\gamma^{-1}-1} \frac{1}{(t+1)^p} \right) \\ 
	&= \frac{1}{(1-\gamma)(t+1)^p}
	\end{align}
	
	From Theorem 2, we have
	\begin{align}
	||V_t - V^*|| &\leq \gamma^t ||V_1 - V^*|| + \frac{\log(|A|)}{(1 - \gamma)(t+1)^p} \\
	& \leq 2\max \{ \gamma^t ||V_1 - V^*||, \frac{\log(|A|)}{(1 - \gamma)(t+1)^p} \}
	\end{align}
	
	Thus, for any $\epsilon > 0$, after at most 
	$t = \max \{\frac{\log(\frac{1}{\epsilon}) + \log(\frac{1}{1 - \gamma}) + \log(R) + \log(4)}{\log(\frac{1}{\gamma})}, {\big( \frac{2 \log(|A|)}{(1 - \gamma) \epsilon} \big)}^{\frac{1}{p}} - 1\}$ steps, we have $||V_t - V^*|| \leq \epsilon$.
\end{proof}

\section{Error Bound of Value Iteration with Fixed Boltzmann Softmax Operator}

\begin{corollary}(Error bound of value iteration with Boltzmann softmax operator)
For any Boltzmann softmax operator with fixed parameter $\beta$, we have
\begin{equation}
\lim_{t\rightarrow \infty}||V_t - V^*||_{\infty} \leq \min\left\{\frac{\log(|A|)}{\beta (1 - \gamma)}, \frac{2R}{{(1-\gamma)}^{2}}\right\}.
\end{equation}
\label{cor:bs_error_bound}
\end{corollary}

\begin{proof}
By Eq. (\ref{eq:vi_con_basic}), it is easy to get that for fixed $\beta$, 

\begin{equation}
\lim_{t\rightarrow \infty}||V_t - V^*||_{\infty} \leq \frac{\log(|A|)}{\beta (1 - \gamma)}.
\end{equation}

On the other hand, we get that 
\begin{align}
	& ||(\mathcal{T}_{\beta} V_1) - (\mathcal{T}_m V_1)||_{\infty} \\
	=& \max_s | {\rm{boltz}}_{\beta} (Q_1(s,\cdot)) - \max_a Q_1(s,a)| \\
	\leq & \max_s |\max_a Q_1(s,a) - \min_a Q_1(s,a) |  \\
	\leq & \max_s |(\max_a r(s,a) - \min_a r(s,a))\\ 
	+ & \gamma (\max_{s'} V_1(s') - \min_{s'} V_1(s'))|\\
	\leq & 2R + \gamma (\max_{s} V_1(s') - \min_{s} V_1(s')).
\label{fix_eq_1}
\end{align}

Combing (\ref{th1_1}), (\ref{th1_3}) and (\ref{fix_eq_1}), we have
	\begin{equation}
	\begin{split}
	&|| (\mathcal{T}_{\beta} V_1) - (\mathcal{T}_m V_2) ||_{\infty} \\
	 \leq &  \gamma ||V_1 - V_2||_{\infty} + 2R + \gamma (\max_{s} V_1(s') - \min_{s} V_1(s')). 
	\end{split}
	\end{equation}
	
Then by the same way in the proof of Theorem 1, 

\begin{align}
	 &||V_t - V^*||_{\infty} \leq  \gamma^t ||V_0 - V^*||_{\infty} \\
	 + & \sum_{k=1}^t \gamma^{t-k} (2R + \gamma (\max_{s} V_{k-1}(s') - \min_{s} V_{k-1}(s'))).
	\label{eq:fix_vi_con_basic}
	\end{align}

Now for the Boltzmann softmax operator, we derive the upper bound of the gap between the maximum value and the minimum value at any timestep $k$.

For any $k$, by the same way, we have
\begin{align}
	& \max_{s} V_{k}(s') - \min_{s} V_{k}(s') \\
	\leq & 2R + \gamma (\max_{s} V_{k-1}(s') - \min_{s} V_{k-1}(s')).
\label{fix_eq_2}
\end{align}

Then by (\ref{fix_eq_2}), 
\begin{equation}
\begin{split}
&\max_{s} V_{k}(s') - \min_{s} V_{k}(s')\\
 \leq & \frac{2R(1-{\gamma}^k)}{1-\gamma} + {\gamma}^{k} (\max_{s} V_{0}(s') - \min_{s} V_{0}(s')).
\end{split}
\label{fix_eq_3}
\end{equation}

Combining (\ref{eq:fix_vi_con_basic}) and (\ref{fix_eq_3}), and Taking the limit, we have

\begin{equation}
\lim_{t\rightarrow \infty}||V_t - V^*||_{\infty} \leq \frac{2R}{{(1-\gamma)}^{2}}.
\end{equation}

\end{proof}

\section{Convergence of DBS Q-Learning}\label{app:q_con}
\textbf{Theorem 3 (Convergence of DBS Q-learning)} \emph{
	The Q-learning algorithm with dynamic Boltzmann softmax policy given by
	\begin{equation}
	\label{eq:q_update}
	\begin{split}
	Q_{t+1}(s_t, a_t) = & (1 - \alpha_t(s_t, a_t)) Q_t(s_t, a_t) + \alpha_t(s_t, a_t) \\
	& [r_t + \gamma {\rm{boltz}}_{\beta_t} (Q_t(s_{t+1}, \cdot))]
	\end{split}
	\end{equation}
	converges to the optimal $Q^*(s,a)$ values if
	\begin{enumerate}
		\item The state and action spaces are finite, and all state-action pairs are visited infinitely often.
		\item $\sum_{t} \alpha_t(s,a) = \infty$ and $\sum_{t} \alpha_t^2(s,a) < \infty$
		\item $\lim_{t \to \infty} \beta_t = \infty$
		\item $\text{Var}(r(s,a))$ is bounded.
	\end{enumerate}
}

\begin{proof}
	Let $\Delta_t(s,a) = Q_t(s,a) - Q^*(s,a)$ and $F_t(s,a) = r_t + \gamma {\rm{boltz}}_{\beta_t} (Q_t(s_{t+1}, \cdot)) - Q^*(s,a)$
	
	Thus, from Eq. (\ref{eq:q_update}) we have
	\begin{equation}
    \Delta_{t+1}(s,a) = (1 - \alpha_t(s,a)) \Delta_t(s,a) + \alpha_t(s,a) F_t(s,a),
	\end{equation}
	which has the same form as the process defined in Lemma 1 in \cite{singh2000convergence}.
	
	Next, we verify $F_t(s,a)$ meets the required properties.
	\begin{align}
	&F_t(s,a) \\
	=& r_t + \gamma {\rm{boltz}}_{\beta_t} (Q_t(s_{t+1}, \cdot)) - Q^*(s,a) \\
	=& \left( r_t + \gamma \max_{a_{t+1}} Q_t(s_{t+1}, a_{t+1}) - Q^*(s,a) \right) + \\
	& \gamma \left( {\rm{boltz}}_{\beta_t} (Q_t(s_{t+1}, \cdot)) - \max_{a_{t+1}} Q_t(s_{t+1}, a_{t+1}) \right) \\
	\overset{\Delta}{=} & G_t(s,a) + H_t(s,a)
	\end{align}
	
	For $G_t$, it is indeed the $F_t$ function as that in Q-learning with static exploration parameters, which satisfies
	\begin{equation}
	|| \mathbb{E} [ G_t(s,a) ] | P_t||_w \leq \gamma || \Delta_t ||_w
	\end{equation}
	
	For $H_t$, we have
	\begin{align}
	& | \mathbb{E} [ H_t(s,a) ] | \\
	=& \gamma \big| \sum_{s'} p(s'|s,a) [{\rm{boltz}}_{\beta_t} (Q_t(s', \cdot)) - \max_{a'} Q_t(s', a')] \big| \\
	\leq & \gamma \big| \max_{s'} [{\rm{boltz}}_{\beta_t} (Q_t(s', \cdot)) - \max_{a'} Q_t(s', a')] \big| \\
	\leq & \gamma \max_{s'} \big| {\rm{boltz}}_{\beta_t} (Q_t(s', \cdot)) - \max_{a'} Q_t(s', a') \big| \\
	\leq & \gamma \max_{s'} \left| {\rm{boltz}}_{\beta_t} (Q_t(s', \cdot)) - L_{\beta_t} (Q_t(s', \cdot)) \right| \\
	\leq & \frac{\gamma \log(|A|)}{\beta_t}
	\end{align}
	
	Let $h_t = \frac{\gamma \log(|A|)}{\beta_t}$, so we have
	\begin{equation}
	|| \mathbb{E} [ F_t(s,a) ] | P_t||_w \leq \gamma || \Delta_t ||_w + h_t,
	\end{equation}
	where $h_t$ converges to $0$.
	
\end{proof}

\section{Analysis of the Overestimation Effect}
\textbf{Proposition 2}
\emph{
	For $\beta_t$ , $\beta > 0$ and M dimensional vector $x$, we have	
	\begin{equation}
\frac{\sum_{i=1}^M e^{\beta_t x_i} x_i}{\sum_{i=1}^M e^{\beta_t x_i}} \leq \max_{i} x_i \leq \frac{1}{\beta} \log \left( \sum_{i=1}^M e^{\beta x_i} \right).
\end{equation}
}

\begin{proof}
	As the dynamic Boltzman softmax operator summarizes a weighted combination of the vector $X$, it is easy to see
	\begin{equation}
	\forall \beta_t>0, \ \frac{\sum_{i=1}^M e^{\beta_t x_i} x_i}{\sum_{i=1}^M e^{\beta_t x_i}} \leq \max_{i} x_i.
	\end{equation}

	Then, it suffices to prove 
	\begin{equation}
	\max_{i} x_i \leq \frac{1}{\beta} \log \left( \sum_{i=1}^M e^{\beta x_i} \right).
	\label{ineq: lse_max}
	\end{equation}

	Multiply $\beta$ on both sides of Ineq. (\ref{ineq: lse_max}), it suffices to prove that
	\begin{equation}
	\max_i \beta x_i \leq \log(\sum_{i=1}^n e^{\beta x_i}).
	\end{equation}
	
	As 
	\begin{equation}
	\max_i \beta x_i = \log(e^{\max_i \beta x_i}) \leq \log(\sum_{i=1}^n e^{\beta x_i}),
	\label{ineq: lse_max_lhs}
	\end{equation}
	Ineq. (\ref{ineq: lse_max}) is satisfied.
	
\end{proof}

\textbf{Theorem 4}  \emph{
	Let $\hat{\mu}^*_{B_{\beta_t}}, \hat{\mu}^{*}_{\max}, \hat{\mu}^{*}_{L_{\beta}}$ denote the estimator with the DBS operator, the max operator, and the log-sum-exp operator, respectively.
	For any given set of $M$ random variables, we have
	$$\forall t, \ \forall \beta, \ \rm{Bias}(\hat{\mu}^*_{B_{\beta_t}}) \leq \rm{Bias}(\hat{\mu}^{*}_{\max}) \leq \rm{Bias}(\hat{\mu}^{*}_{L_{\beta}})$$.
}

\begin{proof}
By definition, the bias of any action-value summary operator $\bigotimes$ is defined as 

\begin{equation}
\label{bias_definition}
\rm{Bias}(\hat{\mu}^*_{\bigotimes}) = \mathbb{E}_{\hat{\mu}\sim \hat{F}}[\bigotimes \hat{\mu}] - \mu^*(X).
\end{equation}

By Proposition 2, we have 

\begin{equation}
\label{expectation}
\begin{split}
\mathbb{E}_{\hat{\mu}\sim \hat{F}}[\frac{\sum_{i=1}^M e^{\beta_t \hat{\mu}_i}\hat{\mu}_i}{\sum_{i=1}^M e^{\beta_t \hat{\mu}_i}} ]  \leq & \mathbb{E}_{\hat{\mu}\sim \hat{F}}[\max_i \hat{\mu}_i]  \\
\leq & \mathbb{E}_{\hat{\mu}\sim \hat{F}}[ \frac{1}{\beta} \log \left( \sum_{i=1}^M e^{\beta  \hat{\mu}_i} \right)] .
\end{split}
\end{equation}

As the ground true maximum value $\mu^*(X)$ is invariant for different operators, combining (\ref{bias_definition}) and (\ref{expectation}), we get

\begin{equation}
\forall t, \ \forall \beta, \ \rm{Bias}(\hat{\mu}^*_{B_{\beta_t}}) \leq \rm{Bias}(\hat{\mu}^{*}_{\max}) \leq \rm{Bias}(\hat{\mu}^{*}_{L_{\beta}}).
\end{equation}

\end{proof}

\section{Empirical Results for DBS Q-learning}

The GridWorld consists of $10 \times 10$ grids, with the dark grids representing walls. The agent starts at the upper left corner and aims to eat the apple at the bottom right corner upon receiving a reward of $+1$. Otherwise, the reward is $0$. An episode ends if the agent successfully eats the apple or a maximum number of steps $300$ is reached.

\begin{figure}[!h]
\centering
\includegraphics[scale=0.26]{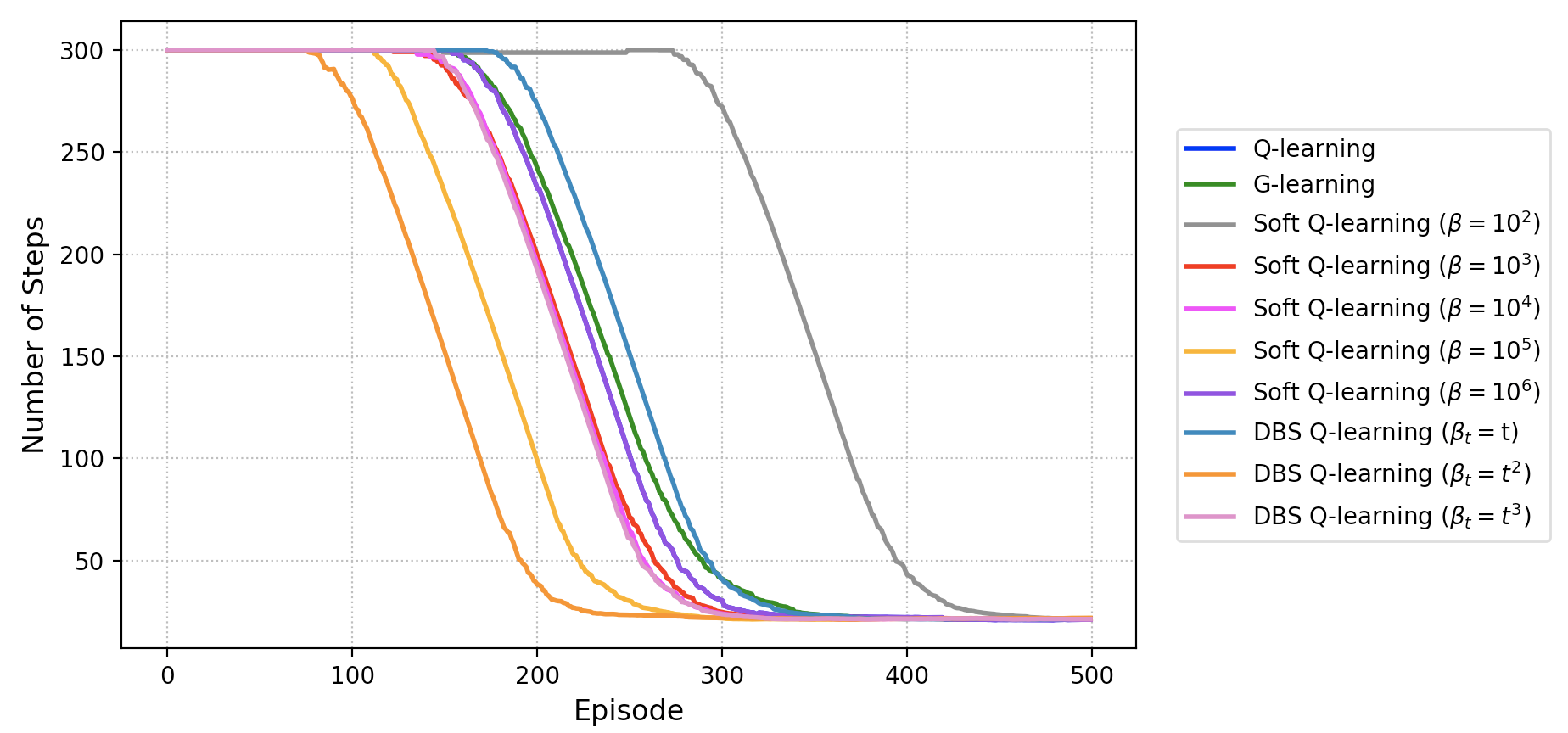}
\label{fig:all_q_res}
\caption{Detailed results in the GridWorld.}
\end{figure}

Figure 1 shows the full performance comparison results among DBS Q-learning, soft Q-learning \cite{haarnoja2017reinforcement}, G-learning \cite{fox2015taming}, and vanilla Q-learning.

As shown in Figure 1, different choices of the parameter of the log-sum-exp operator for soft Q-learning leads to different performance.
A small value of $\beta$ ($10^2$) in the log-sum-exp operator results in poor performance, which is significantly worse than vanilla Q-learning due to extreme overestimation.
When $\beta$ is in $\left\{10^3, 10^4, 10^5\right\}$, it starts to encourage exploration due to entropy regularization, and performs better than Q-learning, where the best performance is achieved at the value of $10^5$.
A too large value of $\beta$ ($10^6$) performs very close to the $\max$ operator employed in Q-learning.
Among all, DBS Q-learning with $\beta=t^2$ achieves the best performance.

\section{Implementation Details}\label{app:implementation_detail}
For fair comparison, we use the same setup of network architectures and hyper-parameters as in \cite{mnih2015human} for both DQN and DBS-DQN.
The network architecture is the same as in (\cite{mnih2015human}).
The input to the network is a raw pixel image, which is pre-processed into a size of $84 \times 84 \times 4$.
Table \ref{table:network_architecture} summarizes the network architecture.

\begin{table*}[h]
\centering
\begin{tabular}{l|l|l|l}
\toprule
\textbf{\textsc{layer}} & \textbf{\textsc{type}} & \textbf{\textsc{configuration}} & \textbf{\textsc{activation}} \\
\hline
\multirow{3}*{1st} & \multirow{3}*{convolutional} & \#filters=32 & \multirow{3}*{ReLU} \\
~ & ~ & size=$8 \times 8$ & ~ \\
~ &~ &stride=$4$  & ~ \\
\hline
\multirow{3}*{2nd} & \multirow{3}*{convolutional} & \#filters=64 & \multirow{3}*{ReLU} \\
~ & ~ & size=$4 \times 4$ & ~ \\
~ &~ &stride=$2$  & ~ \\
\hline
\multirow{3}*{3rd} & \multirow{3}*{convolutional} & \#filters=64 & \multirow{3}*{ReLU} \\
~ & ~ & size=$3 \times 3$ & ~ \\
~ &~ &stride=$1$  & ~ \\
\hline
4th & fully-connected & \#units=512 & ReLU \\
\hline
output & fully-connected & \#units=\#actions & --- \\
\bottomrule
\end{tabular}
\caption{Network architecture.}
\label{table:network_architecture}
\end{table*}

\section{Relative human normalized score on Atari games}
To better characterize the effectiveness of DBS-DQN, its improvement over DQN is shown in Figure \ref{fig: atari_sum}, where the improvement is defined as the relative human normalized score:
\begin{equation}
\frac{\rm{score_{agent}} - {\rm {score_{baseline}}}}{\max\{{\rm score_{human}}, {\rm {score_{baseline}}}\} - {\rm score_{random}}} \times 100\%,
\end{equation}
with DQN serving as the baseline.

\begin{figure}[!h]
    \centering
    \includegraphics[scale=0.41]{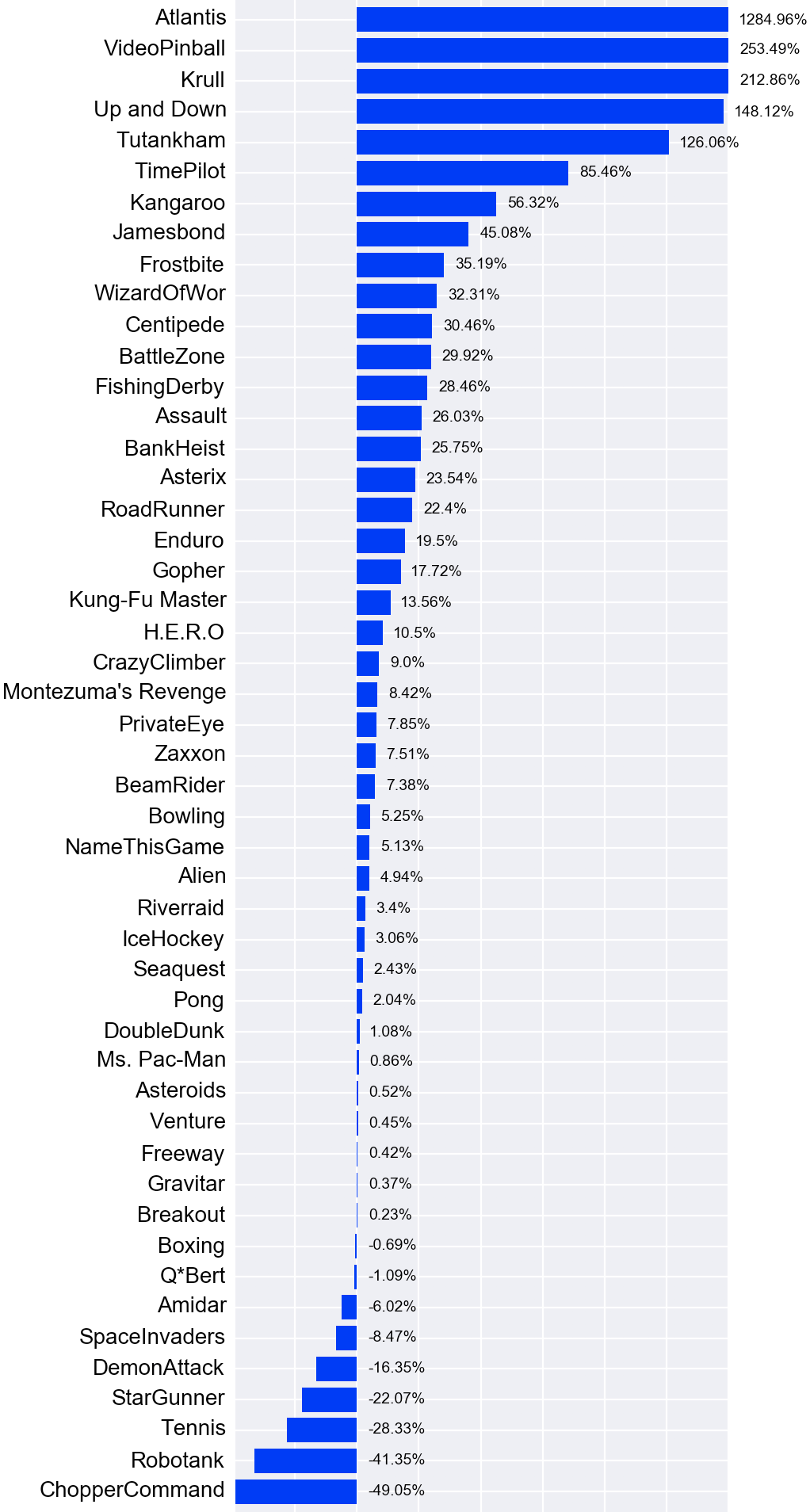}
    \caption{Relative human normalized score on Atari games. }
    \label{fig: atari_sum}
\end{figure}

\section{Atari Scores}
\begin{figure*}[!h]
\small
\centering
\begin{tabular}{ l|l|l|l|l|l}
  \textbf{\textsc{games}}  &  \textbf{\textsc{random}}  &  \textbf{\textsc{human}}  &  \textbf{\textsc{dqn}}  &   \textbf{\textsc{dbs-dqn}} & \textbf{\textsc{dbs-dqn (fixed $c$)}} \\
\hline
Alien & 227.8 & 7,127.7 & 1,620.0 & 1,960.9 & 2,010.4\\
Amidar & 5.8 & 1,719.5 & 978.0 & 874.9 & 1,158.4\\
Assault & 222.4 & 742.0 & 4,280.4 & 5,336.6 & 4,912.8\\
Asterix & 210.0 & 8,503.3 & 4,359.0 & 6,311.2 & 4,911.6\\
Asteroids & 719.1 & 47,388.7 & 1,364.5 & 1,606.7 & 1,502.1\\
Atlantis & 12,850.0 & 29,028.1 & 279,987.0 & 3,712,600.0 & 3,768,100.0\\
Bank Heist & 14.2 & 753.1 & 455.0 & 645.3 & 613.3\\
Battle Zone & 2,360.0 & 37,187.5 & 29,900.0 & 40,321.4 & 38,393.9\\
Beam Rider & 363.9 & 16,926.5 & 8,627.5 & 9,849.3 & 9,479.1\\
Bowling & 23.1 & 160.7 & 50.4 & 57.6 & 61.2\\
Boxing & 0.1 & 12.1 & 88.0 & 87.4 & 87.7\\
Breakout & 1.7 & 30.5 & 385.5 & 386.4 & 386.6\\
Centipede & 2,090.9 & 12,017.0 & 4,657.7 & 7,681.4 & 5,779.7\\
Chopper Command & 811.0 & 7,387.8 & 6,126.0 & 2,900.0 & 1,600.0\\
Crazy Climber & 10,780.5 & 35,829.4 & 110,763.0 & 119,762.1 & 115,743.3\\
Demon Attack & 152.1 & 1,971.0 & 12,149.4 & 9,263.9 & 8,757.2\\
Double Dunk & -18.6 & -16.4 & -6.6 & -6.5 & -9.1\\
Enduro & 0.0 & 860.5 & 729.0 & 896.8 & 910.3\\
Fishing Derby & -91.7 & -38.7 & -4.9 & 19.8 & 12.2\\
Freeway & 0.0 & 29.6 & 30.8 & 30.9 & 30.8\\
Frostbite & 65.2 & 4,334.7 & 797.4 & 2,299.9 & 1,788.8\\
Gopher & 257.6 & 2,412.5 & 8,777.4 & 10,286.9 & 12,248.4\\
Gravitar & 173.0 & 3,351.4 & 473.0 & 484.8 & 423.7\\
H.E.R.O. & 1,027.0 & 30,826.4 & 20,437.8 & 23,567.8 & 20,231.7\\
Ice Hockey & -11.2 & 0.9 & -1.9 & -1.5 & -2.0\\
James Bond & 29.0 & 302.8 & 768.5 & 1,101.9 & 837.5\\
Kangaroo & 52.0 & 3,035.0 & 7,259.0 & 11,318.0 & 12,740.5\\
Krull & 1,598.0 & 2,665.5 & 8,422.3 & 22,948.4 & 7,735.0\\
Kung-Fu Master & 258.5 & 22,736.3 & 26,059.0 & 29,557.6 & 29,450.0\\
Montezuma’s Revenge & 0.0 & 4,753.3 & 0.0 & 400.0 & 400.0\\
Ms. Pac-Man & 307.3 & 6,951.6 & 3,085.6 & 3,142.7 & 2,795.6\\
Name This Game & 2,292.3 & 8,049.0 & 8,207.8 & 8,511.3 & 8,677.0\\
Pong & -20.7 & 14.6 & 19.5 & 20.3 & 20.3\\
Private Eye & 24.9 & 69,571.3 & 146.7 & 5,606.5 & 2,098.4\\
Q*Bert & 163.9 & 13,455.0 & 13,117.3 & 12,972.7 & 10,854.7\\
River Raid & 1,338.5 & 17,118.0 & 7,377.6 & 7,914.7 & 8,138.7\\
Road Runner & 11.5 & 7,845.0 & 39,544.0 & 48,400.0 & 44,900.0\\
Robotank & 2.2 & 11.9 & 63.9 & 42.3 & 41.9\\
Seaquest & 68.4 & 42,054.7 & 5,860.6 & 6,882.9 & 6,974.8\\
Space Invaders & 148.0 & 1,668.7 & 1,692.3 & 1,561.5 & 1,311.9\\
Star Gunner & 664.0 & 10,250.0 & 54,282.0 & 42,447.2 & 38,183.3\\
Tennis & -23.8 & -8.3 & 12.2 & 2.0 & 2.0\\
Time Pilot & 3,568.0 & 5,229.2 & 4,870.0 & 6,289.7 & 6,275.7\\
Tutankham & 11.4 & 167.6 & 68.1 & 265.0 & 277.0\\
Up and Down & 533.4 & 11,693.2 & 9,989.9 & 26,520.0 & 20,801.5\\
Venture & 0.0 & 1,187.5 & 163.0 & 168.3 & 102.9\\
Video Pinball & 16,256.9 & 17,667.9 & 196,760.4 & 654,327.0 & 662,373.0\\
Wizard Of Wor & 563.5 & 4,756.5 & 2,704.0 & 4,058.7 & 2,856.3\\
Zaxxon & 32.5 & 9,173.3 & 5,363.0 & 6,049.1 & 6,188.7
\end{tabular}
\caption{Raw scores for a single seed across all games, starting with 30 no-op actions. Reference values from \cite{wang2015dueling}.\label{fig:atari_sota}}
\end{figure*}

\clearpage
\bibliography{example_paper}
\bibliographystyle{example_paper}

\end{document}